\documentclass[onefignum,onetabnum]{siamart251216}

\usepackage[utf8]{inputenc}
\usepackage[T2A]{fontenc}
\usepackage{amsmath,amssymb,amscd,amsopn}
\usepackage{mathtools}
\usepackage{algorithm}
\usepackage{algpseudocode}
\usepackage{algorithmicx}
\usepackage{booktabs}
\usepackage{multirow}
\usepackage{graphicx}
\usepackage{epstopdf}

\newtheorem{remark}{Remark}

\ifpdf
\hypersetup{
  pdftitle={Riemannian Optimization for a Class of Low-Parametric Orthogonal Matrices},
  pdfauthor={Ali Aliev, Maxim Rakhuba}
}
\fi

\headers{Riemannian Optimization for \texorpdfstring{\ensuremath{\mathcal{GS}}}{GS} Orthogonal Matrices}{A. Aliev, M. Rakhuba}

\title{Riemannian Structure and Optimization for a Class of Low-Parametric Orthogonal Matrices}

\author{Ali Aliev\thanks{HSE University (\email{alievali0278@gmail.com}).}
\and Maxim Rakhuba\thanks{HSE University.}}

\begin{document}

\maketitle

\begin{abstract}
In this paper, we are concerned with matrices formed by block-diagonal factors interleaved with fixed permutations -- a flexible family of structured matrices. This class has recently drawn interest in deep learning architectures for its balanced expressivity-efficiency trade-off, yet efficient computational strategies for working with it remain to be found. We approach this problem through Riemannian geometry and examine under what conditions this class admits a smooth manifold structure. For the practically important case of orthogonal two-factor matrices, we derive the essential Riemannian tools and propose efficient algorithms for their implementation. The algorithms leverage automatic differentiation, support parameter sharing within each factor, and avoid explicit dense matrix construction. We test them within the Riemannian optimization framework on the best matrix approximation problem and for parameter-efficient fine-tuning of large language models. Beyond the two-factor setting, we study the geometric and matrix-theoretic properties of factorizations with a larger number of block-diagonal factors.  
\end{abstract}



\section{Introduction}
Structured matrix approximations are essential in numerical linear algebra and large-scale computing. They provide a practical way to balance an operator's expressiveness with its computational cost. Classical structures, such as Toeplitz \cite{Gray2006}, circulant \cite{Davis1979}, and Cauchy matrices \cite{Pan2001Structured}, allow for nearly linear matrix-vector multiplications. There are more recent examples, including butterfly matrix factorizations ~\cite{li2015butterfly, candes2009fast} that have a rich history, originally arising from the structure of fast unitary transforms like the Fast Fourier Transform (FFT) \cite{Parker1995}. 

While well-established in numerical analysis, butterfly matrix factorizations have recently gained popularity in deep learning for model compression \cite{Dao2019, Dao22}. A natural generalization of these butterfly-based and block-diagonal constructions is the \textit{Group-and-Shuffle} (\(\mathcal{GS}\)) family \cite{GS}. This class unifies several previous approaches, such as Monarch matrices \cite{Dao22} and block-butterfly factorizations \cite{Dao2019, Parker1995}, by parameterizing matrices as alternating sequences of block-diagonal factors and specific permutations.

Euclidean optimization of \(\mathcal{GS}\)-matrices ignores their geometry, which remains mainly unstudied despite empirical success (e.g., orthogonal fine-tuning \cite{GS}). In contrast, Riemannian optimization methods explicitly exploit the manifold geometry of the feasible set, offering a widely adopted framework for constrained problems \cite{Absil2008}.

This work utilizes a definition expressed in an alternative notation, introducing a minor formal distinction from the original work \cite{GS}.
\begin{definition} \label{GS_main_definition}
A matrix \( A \) is said to be in the class \( \mathcal{A}_m(P_{m-1}, \ldots, P_1) \) if it can be factorized as
\[
A = B_m P_{m-1} B_{m-1} \dots P_1 B_1,
\]
where:
each \( B_i \) is a block-diagonal matrix with \( k_i \) blocks of size \( b_i^1 \times b_i^2 \), each \( P_i \) is a fixed permutation matrix, and the dimensions are compatible according to the condition \(b_i^1 \cdot k_i = b_{i+1}^2 \cdot k_{i+1}.\)
\end{definition}
In what follows, we will simply write \(\mathcal{A}_m\) for brevity. Analogous sets \(\mathcal{A}_m^{\mathrm{orth}}\) can be introduced by imposing the additional condition that the blocks in each \(B_i\) be orthogonal. Similarly, we can define \(\mathcal{A}_m^{\mathrm{SO}}, \mathcal{A}_m^{\times}\), \(\mathcal{A}_m^{\mathrm{unit}}\), and \(\mathcal{A}_m^{\mathrm{SU}}\) corresponding to special orthogonal, invertible, unitary, and special unitary blocks in \(B_i\), respectively.  

In this paper, we focus on the case \(m = 2\). This setting satisfies the structural property \(\mathcal{A}_2^{\mathrm{orth}} = \mathcal{A}_2 \cap \mathrm{O}(N)\) established in \cite[Theorem 1]{GS}, while also being empirically successful in modern deep learning and computationally friendly (see \cite{GS, OF, Dao22}). Indeed, although near-linear complexity is desirable in many classical settings, representations used in deep learning often benefit from avoiding recursive structures and excessive sequential operations, especially because of backpropagation through long computational graphs. Moreover, modern GPUs are highly optimized for dense matrix multiplication, and the product of several moderately sized matrices can be more efficient than dozens of multiplications involving much smaller factors. This effect is likely to become even more pronounced as GPU architectures and matrix-multiplication kernels continue to improve.

\paragraph{1.1. Contributions}
This paper provides a geometric and algorithmic analysis of the \(\mathcal{GS}\) matrix family. Our main contributions are:
\begin{enumerate}
    \item We show by explicit counterexamples that unconstrained block factorizations with rectangular or square blocks need not form smooth manifolds.
    \item For \(\mathcal{GS}\)-matrices with orthogonal blocks, whose smooth manifold structure was recently established, we derive practical first‑order Riemannian optimization algorithms.
    \item We propose to use automatic differentiation (AD) to compute the tangent space projection while naturally handling overlapping block supports, without materializing dense matrices. As an application of this technique, we consider an efficient Riemannian Gradient Descent (RGD) algorithm and apply it to a matrix approximation problem and fine-tuning of LLMs.
    \item We further extend our framework to support parameter sharing. This enables additional compression and speedup by adapting our efficient automatic differentiation mechanism.
    \item We study geodesics for the case \(m=2\) under additional assumptions on the blocks, and for \(m=3\) we establish a connection with the CS decomposition.
    \item We prove that \(\mathcal{A}_3^{\mathrm{orth}} = \mathcal{A}_3 \cap \mathrm{O}(N)\), and we discuss hierarchical and higher-order factorizations.
\end{enumerate}

\paragraph{1.2. Structure of the paper}
The rest of the paper is organized as follows. Section~\ref{sec:manifold_conditions} investigates the manifold properties of the \(\mathcal{GS}\) family, presenting counterexamples for rectangular and general square blocks. Section \ref{sec:optimization} focuses on the orthogonal case: it derives the Riemannian geometry, details the retraction map in Section \ref{subsec:retraction}, and describes the tangent space projection in Section \ref{subsec:proj_tangent}. Section \ref{subsec:ad_implementation} provides the efficient gradient computation via automatic differentiation, while Section \ref{subsec:sharing} develops the shared-parameter version of the projection and AD formulas. Finally, Section~\ref{sec:experiments} presents numerical experiments that validate our framework. In the appendix, geodesics are studied under certain conditions on the blocks, along with the specific case \(m = 3\), for which interesting properties are revealed. Additionally, the invertible case is briefly discussed.

\paragraph{Related Work}
Structured matrices with subquadratic matrix-vector multiplication costs (such as sparse, low-rank, and fast transforms) are classic topics in numerical linear algebra \cite{Higham2002}. Butterfly matrices \cite{Parker1995} were originally used to describe fast algorithms like the FFT. Recently, they have been used to replace dense layers in neural networks \cite{Sindhwani2015, Dao2019}. Monarch matrices \cite{Dao22} combine block-diagonal factors with permutations, and the \(\mathcal{GS}\) class \cite{GS} generalizes these ideas by allowing flexible block sizes and varying numbers of factors. 

Optimization over constrained spaces, such as orthogonal groups, can be carried out using manifold optimization techniques \cite{Absil2008}. Matrix manifolds help preserve gradient norms and control Lipschitz constants \cite{Li2019, Singla2021}. Recently, orthogonal parameterizations have been applied to Parameter-Efficient Fine-Tuning (PEFT). While additive low-rank updates (like LoRA \cite{Hu2022}) are popular, multiplicative orthogonal updates (like OFT \cite{Qiu2023} and BOFT \cite{Liu2024}) have shown better results in preserving the pretrained structure of weight matrices.

\section{Conditions for the Smooth Manifold Structure of \texorpdfstring{\(\mathcal{GS}\)}{GS}-matrices} \label{sec:manifold_conditions}

It is natural to assume that for each admissible block type, the structured parametrization of the $\mathcal{GS}$ family yields a smooth manifold. However, in this section, we show that this is not a given: without specific constraints, the general set $\mathcal{A}_m$ fails to form a smooth manifold. We demonstrate this for $m=2$ by relying on the following structural property of $\mathcal{GS}$ block multiplication.

\begin{proposition}[Proposition 1 in \cite{GS}] \label{prop:gs_expansion}
Let $A = B_2 P B_1$ where the right factor $B_1$ consists of $k_1$ blocks of size $b_1 \times b_1$,
the left factor $B_2$ consists of $k_2$ blocks of size $b_2 \times b_2$, 
and $P$ is a permutation matrix with index map $\sigma$.
Denote by $\{v_i^\top\}$ the rows of the blocks of $B_1$ and by $\{u_j\}$ the columns of the blocks of $B_2$, 
both in consecutive order. 
Then $A$ can be written as a block matrix with $k_2 \times k_1$ blocks, where the block $A_{\alpha,\beta}$ is given by
\begin{equation} \label{eq:outer_product_sum}
    A_{\alpha,\beta} = \sum_{\substack{\lfloor \sigma(i)/b_2 \rfloor = \alpha \\[2pt] \lfloor i/b_1 \rfloor = \beta}} u_{\sigma(i)} v_i^\top .
\end{equation}
(Zero‑indexing is used for simplicity.)
\end{proposition}

We explicitly apply this to the Perfect Shuffle permutation $P_{(b, k)}$ matrices~\cite{GL}.

\begin{definition}
Let \(N = b\cdot k\). For a zero-based index
\(i = r\cdot b+s\), with \(0\le r < k\) and \(0\le s < b\), define the perfect shuffle permutation by
\[
\sigma_{(b,k)}(r \cdot b+s) = s \cdot k+r.
\]
Let \(P_{(b,k)}\in\mathbb{R}^{N\times N}\) be the corresponding
permutation matrix, defined by $P_{(b,k)}e_i=e_{\sigma_{(b,k)}(i)}.$ The inverse permutation is \(P_{(b,k)}^{-1} = P_{(b,k)}^\top = P_{(k,b)}\).
\end{definition}

\subsection{Arbitrary Rectangular and Square Blocks}
We first show by an explicit counterexample that allowing arbitrary rectangular and even arbitrary square blocks does not, in general, yield a smooth manifold.

Set $k_L = k_R = 2$. Let $B_1$ consist of two $2 \times 3$ blocks with row vectors $u_1, u_2 \in \mathbb{R}^{1 \times 3}$ (block 0) and $v_1, v_2 \in \mathbb{R}^{1 \times 3}$ (block 1). Let $B_2$ consist of two $3 \times 2$ blocks with column vectors $x_1, x_2 \in \mathbb{R}^{3 \times 1}$ (block 0) and $y_1, y_2 \in \mathbb{R}^{3 \times 1}$ (block 1). 

Here, $P$ is the $4 \times 4$ Perfect Shuffle $P_{(2,2)}$. By its reshaping definition, it maps indices $(0, 1, 2, 3)$ to $\sigma = (0, 2, 1, 3)$. 
Let us evaluate block $A_{0,0}$ ($k_1=0, k_2=0$). The condition $\lfloor i/2 \rfloor = 0$ restricts $i \in \{0, 1\}$. The second condition $\lfloor \sigma(i)/2 \rfloor = 0$ is satisfied only by $i=0$ (since $\sigma(0)=0$, whereas $\sigma(1)=2$). Thus, the sum in \eqref{eq:outer_product_sum} collapses to a single term: $A_{0,0} = u_{\sigma(0)} v_0^\top = x_1 u_1$.
Applying this logic to all blocks yields exactly:
\begin{equation}
    A = \begin{pmatrix} x_1 u_1 & x_2 v_1 \\ y_1 u_2 & y_2 v_2 \end{pmatrix}.
\end{equation}
Each of the four blocks is a single outer product. Since the parameter vectors are unconstrained, they can be arbitrary or zero, meaning both rank 1 and rank 0 are attainable. Thus, $A$ is parameterized in the Cartesian product of four sets of $3 \times 3$ matrices of rank $\le 1$.

This rank collapse also occurs with square blocks. Let $B_1$ have $k_R = 2$ square $3 \times 3$ blocks with rows $u_1, u_2, u_3$ and $v_1, v_2, v_3 \in \mathbb{R}^{1 \times 3}$. Let $B_2$ have $k_L = 3$ blocks of size $2 \times 2$ with columns $x_1, x_2$, $y_1, y_2$, and $z_1, z_2 \in \mathbb{R}^{2 \times 1}$. 

Here, $P$ is the $6 \times 6$ Perfect Shuffle $P_{(3,2)}$, which maps $i \in \{0,\dots,5\}$ to $\sigma(i) \in \{0, 2, 4, 1, 3, 5\}$. We evaluate the conditions $\lfloor \sigma(i)/2 \rfloor = k_1 \in \{0,1,2\}$ and $\lfloor i/3 \rfloor = k_2 \in \{0,1\}$. 
For example, for the block $(k_1, k_2) = (1, 0)$, the condition $\lfloor i/3 \rfloor = 0$ implies $i \in \{0,1,2\}$. Their corresponding $\sigma(i)$ are $\{0, 2, 4\}$. Among these, only $\sigma(1)=2$ satisfies $\lfloor \sigma(i)/2 \rfloor = 1$. Thus, exactly $i=1$ is selected, yielding $A_{1,0} = u_2 v_1^\top = y_1 u_2$. Evaluating all six $(k_1, k_2)$, we derive:
\begin{equation}
    A = \begin{pmatrix}
    x_1 u_1 & x_2 v_1 \\
    y_1 u_2 & y_2 v_2 \\
    z_1 u_3 & z_2 v_3
    \end{pmatrix} \in \mathbb{R}^{6 \times 6}.
\end{equation}
Again, each of the six $2 \times 3$ blocks is a single outer product, has rank at most one.

In both cases, $\mathcal{A}_2$ is a Cartesian product of bounded-rank sets. These sets fail to be smooth manifolds. A proof proceeds by contradiction, using the results of \cite{Olikier2026}; the details are left to the reader.

Throughout the rest of this work, we consider only square blocks.

\subsection{Unitary or Orthogonal blocks}

For $m=2$ orthogonal and unitary factors give smooth embedded manifolds for every fixed permutation. This can be proved using Lie group action theory; details can be found in \cite{OF}. We briefly recall the argument for orthogonal blocks. Let
\[
\mathcal{B}_1 = \mathrm{O}(n_1) \times \dots \times \mathrm{O}(n_{l_1}), \qquad
\mathcal{B}_2 = \mathrm{O}(m_1) \times \cdots \times \mathrm{O}(m_{l_2})
\]
and set \(G = \mathcal{B}_2 \times \mathcal{B}_1\). Note that \(G\) is a compact Lie group. Observe that the orthogonality condition ensures that the transpose operation preserves both the block-diagonal structure and orthogonality. This allows us to define a Lie group action as follows:
\[
\Phi: (\mathcal{B}_2 \times \mathcal{B}_1) \times M \to M, \qquad \Phi\left((B_2, B_1), x\right) = g \cdot x = B_2 x B_1^{\top}.
\]
Since the orbit of the action of a compact Lie group on a manifold (either \(\mathrm{GL}_N(\mathbb{R})\) or \(\mathrm{O}(N)\)) yields a closed embedded submanifold (see \cite[Theorem 2.3]{GOV1993} or \cite[Corollary 21.6 \& Problem 21-17]{Lee}), we obtain the desired result. 

For \(m = 3\), there is considerable freedom in selecting block sizes and permutation matrices. Appendix \ref{appendix:dimension_m3} focuses on the case of Perfect Shuffle permutation matrices and relates it to the classical Cosine-Sine (CS) decomposition.

Invertible blocks (instead of orthogonal) are also of interest; see Appendix \ref{appendix:invertible} for a brief discussion.

\section{Riemannian Optimization Framework for Orthogonal Blocks}
\label{sec:optimization}

We develop an efficient Riemannian optimization framework for \(\mathcal{GS}\)-orthogonal matrices, including retraction, tangent projection, automatic differentiation (AD), and parameter sharing.  
Recall that for minimizing an objective function \(f\) on a smooth manifold \(\mathcal{M}\), the Riemannian gradient is \(\operatorname{grad} f(X) = \operatorname{P}_{T_X\mathcal{M}}(\nabla f(X))\), where \(\operatorname{P}_{T_X\mathcal{M}}\) is the orthogonal projection onto the tangent space. Parameter updates are performed using a retraction --- a smooth map \( R : T\mathcal{M} \to \mathcal{M} \) satisfying \( R_X(0_X) = X \) and \( DR_X(0_X) = \text{id}_{T_X\mathcal{M}} \) (see \cite[Definition 4.1.1]{Absil2008}) — which ensures that the updated point remains on the manifold, e.g., \( X_{k+1} = R_{X_k}(-\tau_k \operatorname{grad} f(X_k)) \) in the GD case.
To avoid forming the full gradient, we use an AD trick: introducing an auxiliary function on the tangent space and differentiating at zero yields the Riemannian gradient without dense matrix operations.

\subsection{Projection onto the Tangent Space}  \label{subsec:proj_tangent} 

While the smoothness of the manifold holds for any fixed permutation matrix, in the remainder of this paper we specifically take \(P_1\) to be the perfect shuffle permutation \(P_{(b_1, k_1)}\) (unless the context indicates otherwise), due to its notable properties,  that are particularly beneficial in the context of neural networks, as discussed in \cite{GS, OF, Dao22}. Moreover, such a permutation maximizes the manifold dimension for fixed block sizes satisfying \(k_2 \ge b_1\) ~\cite[Appendix B]{OF}.

We seek the projection onto the tangent space. For this purpose, we will employ the following property of such permutation matrices. 
\begin{lemma}[\cite{Dao22, GL}] \label{similarity_PS}
Let \(P\) be the Perfect Shuffle permutation matrix of order \(N = p \cdot q\). For any matrices \(U\) and \(W\) of sizes \(q \times q\) and \(p \times p\), respectively, the following identity holds
\[
P (U \otimes W) P^{\top} = W \otimes U.
\]
\end{lemma}

\begin{corollary} \label{corollary:similarity_PS}
Let \(Q = \operatorname{diag}(U_1, \dots, U_p)\) be a block diagonal matrix with blocks \(U_k\) of size \(q \times q\). Consider the similarity transformation
\[
M = P^{\top} Q P,
\]
where \(P\) is the perfect shuffle matrix from Lemma \ref{similarity_PS}. Then
\[
M = \begin{pmatrix}
D_{11} & \cdots & D_{1q} \\
\vdots & \ddots & \vdots \\
D_{q1} & \cdots & D_{qq}
\end{pmatrix}, \qquad
D_{ij} = \operatorname{diag}(u_{ij}^{(1)}, u_{ij}^{(2)}, \dots, u_{ij}^{(p)}).
\]
\end{corollary}

\begin{proof}
Let \(E_{kk}\) denote the \(p \times p\) matrix with \(1\) at the \((k,k)\) entry and zeros elsewhere. Note that \(P^\top\) is also a perfect shuffle permutation matrix. Applying Lemma \ref{similarity_PS} with \(P^\top\) in place of \(P\) yields \(P^\top (U \otimes W) P = W \otimes U\) for any matrices \(U,W\) of appropriate sizes. Using this, we obtain
\[
M = \sum_{k=1}^p P^\top (E_{kk} \otimes U_k) P = \sum_{k=1}^p U_k \otimes E_{kk}.
\]
The term \(U_k \otimes E_{kk}\) is a \(q \times q\) block matrix whose \((i,j)\)-th block is \(u_{ij}^{(k)} E_{kk}\). Summing over \(k\), the \((i,j)\)-th block of \(M\) becomes
\[
D_{ij} = \sum_{k=1}^p u_{ij}^{(k)} E_{kk} = \operatorname{diag}(u_{ij}^{(1)}, u_{ij}^{(2)}, \dots, u_{ij}^{(p)}),
\]
which completes the proof.
\end{proof}

For projecting Euclidean gradients onto the tangent space \(T_X \mathcal{A}_2^{\mathrm{orth}}\), we will use the standard fact (see, e.g., \cite[Example 3.5.3]{Absil2008}) that the tangent space to the orthogonal group is given by:
\[
T_{U}\mathrm{O}_{n} = \{Z = U\Omega : \Omega^{T} = -\Omega\} = U\mathcal{S}_{\mathrm{skew}}(n),
\]  
where \(\mathcal{S}_{\mathrm{skew}}(n)\) denotes the set of all skew‑symmetric \(n\times n\) matrices.

We introduce the space of block‑diagonal skew‑symmetric matrices (associated with \(B_i\)), denoted \(\mathfrak{so}_{\text{block}}(b_i, k_i)\). Similarly, let $\mathcal{P}_{\text{bdiag}(b_i, k_i)}$ denote the orthogonal projection operator that extracts the $k_i$ block-diagonal components of size $b_i \times b_i$ from an arbitrary matrix, setting all off-block-diagonal entries to zero.

Now let \(C_i\) be a block‑diagonal skew‑symmetric matrix. Then the tangent vectors to \(\mathcal{B}_i\) are precisely the matrices \(B_i C_i\). For convenience, define \( \Omega_i = B_i C_i B_i^\top \iff B_i C_i = \Omega_i B_i \). Note that \(\Omega_i\) is also block-diagonal with skew-symmetric blocks (\(\Omega_i^\top = (B_i C_i B_i^\top)^\top = B_i (-C_i) B_i^\top = -\Omega_i\)).
Given a Euclidean gradient \(Z\), the projection is found by solving
\begin{equation*}
\begin{split}
\Pi_X(Z) &= \operatorname*{argmin}_{\xi \in T_X \mathcal{GS}} \| Z - \xi \|_F^2 \\
&= \operatorname*{argmin}_{C_1, C_2} \| Z - (B_2 C_2 P_1 B_1 + B_2 P_1 B_1 C_1) \|_F^2 \\
&= \operatorname*{argmin}_{C_1, C_2} \| B_2^\top Z B_1^\top - (C_2 P_1 + P_1 \Omega_1) \|_F^2 \\
&= \operatorname*{argmin}_{C_1, C_2} \| \widetilde{Z} - (C_2 + \underbrace{P_1 \Omega_1 P_1^\top}_{\widehat{\Omega}_1}) \|_F^2, \quad \text{where } \widetilde{Z} := B_2^\top Z B_1^\top P_1^\top.
\end{split}
\end{equation*}
Using the decomposition \(\widetilde{Z} = \operatorname{sym}(\widetilde{Z}) + \operatorname{skew}(\widetilde{Z})\), the objective simplifies:
\begin{equation*}
\begin{split}
\| \widetilde{Z} - (\widehat{\Omega}_1 + C_2) \|_F^2 &= \| \operatorname{sym}(\widetilde{Z}) + \operatorname{skew}(\widetilde{Z}) - (\widehat{\Omega}_1 + C_2) \|_F^2 \\
&= \|\operatorname{sym}(\widetilde{Z})\|_F^2 + \|\operatorname{skew}(\widetilde{Z}) - (\widehat{\Omega}_1 + C_2)\|_F^2,
\end{split}
\end{equation*}
since symmetric and skew-symmetric matrices are orthogonal with respect to the Frobenius inner product, and both \(\widehat{\Omega}_1\) and \(C_2\) are skew-symmetric. Consequently, it suffices to solve
\begin{equation*}
\min_{\substack{C_2 \in \mathfrak{so}_{\text{block}}(b_2, k_2) \\ \Omega_1 \in \mathfrak{so}_{\text{block}}(b_1, k_1)}} \Big\| \widehat{Z} - \left(C_2 + \widehat{\Omega}_1 \right) \Big\|_F^2, \quad \widehat{Z} := \operatorname{skew}(\widetilde{Z}).
\end{equation*}

Let us partition \(\widehat{Z}\) into blocks according to the splitting of \(B_2\). Since \(\Omega_1\) is block-diagonal with skew-symmetric blocks, by choosing \(P_1\) as a Perfect Shuffle matrix and using Lemma \ref{similarity_PS}, we obtain that \(P_1 \Omega_1 P_1^\top\) is a block matrix with blocks of size \(k_1 \times k_1\) (there are \(b_1 \times b_1\) such blocks in total), where each block is diagonal. Moreover, skew-symmetry implies that the diagonal blocks on the main diagonal have zero entries, meaning the diagonal blocks are identically zero matrices. This allows us to rewrite our minimization task as follows:
\begin{equation} \label{minimization_tangent_project}
\min_{\substack{C_2, \Omega_1}} \left\|
\begin{pmatrix}
\widehat{Z}^{(11)} - C_2^{(1)} & \dots & \widehat{Z}^{(1 k_2)} \\
\vdots & \ddots & \vdots \\
\widehat{Z}^{(k_2 1)} & \dots & \widehat{Z}^{(k_2 k_2)} - C_{2}^{(k_2)}
\end{pmatrix}
- 
\begin{pmatrix}
0 & \widehat{\Omega}_1^{(1 2)} & \dots & \widehat{\Omega}_1^{(1 b_1)} \\
\vdots & 0 & \dots & \vdots \\
\vdots & \vdots & \ddots & \vdots \\
\widehat{\Omega}_1^{(b_1 1)} & \dots & \dots & 0
\end{pmatrix} \right\|_F^2,   
\end{equation}
where \(\widehat{\Omega}_1^{(\alpha\beta)} = \operatorname{diag}\left(\omega_1^{(\alpha\beta)},\dots,\omega_{k_1}^{(\alpha\beta)}\right),
\quad \omega_i^{(\beta\alpha)} = -\omega_i^{(\alpha\beta)}.\)
The matrices \(C_2\) and \(\Omega_1\) are skew-symmetric. We showed earlier:
\[ \Omega_1^\top = -\Omega_1 \iff (P_1 \Omega_1 P_1^\top)^\top = P_1 \Omega_1^\top P_1^\top = -P_1 \Omega_1 P_1^\top. \]
Thus, the minimization with respect to skew-symmetric matrices \((C_2, \Omega_1)\) is equivalent to the minimization with respect to skew-symmetric matrices \((C_2, \widehat{\Omega}_1)\) under the skew-symmetry constraints. The structure of the solution depends on the relationship between the block count \( k_1 \) and the block size \( b_2 \).

\paragraph{Case 1: Disjoint Supports (\( k_1 \ge b_2 \))}

In this case the matrix \(\widehat{\Omega}_1\) does not affect the diagonal blocks, and the matrix \(C_2\) does not affect the off-diagonal blocks; therefore, we can independently optimize first the diagonal blocks. More formally
\begin{equation*}
\begin{split}
\| \widehat{Z} - (C_2 + \widehat{\Omega}_1) \|_F^2 &= \| \widehat{Z} \|_F^2 - 2 \langle \widehat{Z}, C_2 + \widehat{\Omega}_1 \rangle + \| C_2 + \widehat{\Omega}_1 \|_F^2 \\
&= \| \widehat{Z} \|_F^2 - 2 \langle \widehat{Z}, C_2 \rangle - 2 \langle \widehat{Z}, \widehat{\Omega}_1 \rangle + \left( \|C_2\|_F^2 + \|\widehat{\Omega}_1\|_F^2 + 2 \langle C_2, \widehat{\Omega}_1 \rangle \right).
\end{split}
\end{equation*}

A crucial observation is that the subspaces of the parameters are orthogonal, i.e., \(\langle C_2, \widehat{\Omega}_1 \rangle = 0\). Indeed, \(C_2\) is strictly block-diagonal with blocks of size \(b_2\). Under the dimension constraint \(k_1 \geq b_2\) (which is equivalent to $k_2 \geq b_1$), the matrix \(\widehat{\Omega}_1\) has zero entries in the off-diagonal positions of the \(k_2 \times k_2\) blocks. The only possible overlap of non-zero patterns is on the main diagonal, but since both matrices are skew-symmetric, their diagonals are strictly zero. Thus, the matrices have disjoint supports, implying:
\[
\langle C_2, \widehat{\Omega}_1 \rangle = \sum_{i,j} (C_2)_{ij} (\widehat{\Omega}_1)_{ij} = 0.
\]
The optimization problem therefore decouples into two independent sub-problems (by completing the squares for each variable):
\begin{equation*}
\begin{split}
\| \widehat{Z} - (C_2 + \widehat{\Omega}_1) \|_F^2 &= \left( \|C_2\|_F^2 - 2 \langle \widehat{Z}, C_2 \rangle \right) + \left( \|\widehat{\Omega}_1\|_F^2 - 2 \langle \widehat{Z}, \widehat{\Omega}_1 \rangle \right) + \| \widehat{Z} \|_F^2 \\
&= \| \widehat{Z} - C_2 \|_F^2 + \| \widehat{Z} - \widehat{\Omega}_1 \|_F^2 - \|\widehat{Z}\|_F^2.
\end{split}
\end{equation*}
We can rewrite the optimal points in a slightly different form. For \(C_2\), we simply project \(\widehat{Z}\) onto the subspace of block-diagonal matrices with block size \(b_2\):
\begin{equation*}
C_2 = \mathcal{P}_{\text{bdiag}(b_2, k_2)}(\widehat{Z}).
\end{equation*}

To find \(\Omega_1\), we consider the term \(\| \widehat{Z} - P_1 \Omega_1 P_1^\top \|_F^2\). Exploiting the invariance of the Frobenius norm under orthogonal transformations, we apply the similarity transformation defined by \(P_1^\top (\cdot) P_1\):
\begin{equation*}
\| \widehat{Z} - P_1 \Omega_1 P_1^\top \|_F^2 = \| P_1^\top \widehat{Z} P_1 - P_1^\top (P_1 \Omega_1 P_1^\top) P_1 \|_F^2 = \| P_1^\top \widehat{Z} P_1 - \Omega_1 \|_F^2.
\end{equation*}
Since \(\Omega_1\) is required to be block-diagonal with block size \(b_1\), the solution is given explicitly by projecting the permuted matrix:
\begin{equation*}
\Omega_1 = \mathcal{P}_{\text{bdiag}(b_1, k_1)} \left( P_1^\top \widehat{Z} P_1 \right).
\end{equation*}

Observe that in \cite{OF}, the same condition that yields the disjoint support case is exactly the one that maximizes the orbit dimension and minimizes (more precisely, nullifies) the dimension of the stabilizer. Moreover, \cite{OF} uses this condition to provide some empirical justification for finding geodesics in the space of \(\mathcal{GS}\) matrices with blocks from the special orthogonal group. In Appendix~\ref{appendix:geodesics}, we present a rigorous argument for this approach.

\paragraph{Case 2: Overlapping Supports (\( k_1 < b_2 \))}
In the case where \(k_1 < b_2\), the structural supports of the parametric matrices \(C_2\) and \(\widehat{\Omega}_1 := P_1 \Omega_1 P_1^\top\) overlap. To properly project the gradient \(\widehat{Z}\) onto the tangent space, we solve the minimization problem \eqref{minimization_tangent_project} index-wise based on the respective supports.

Let \(\mathcal{S}_{C}\) be the set of matrix indices \((i,j)\) where \(C_2\) can be non-zero, and let \(\mathcal{S}_{\Omega}\) be the support of \(\widehat{\Omega}_1\). The explicit set of overlapping indices is given by their intersection: \(\mathcal{S}_{\text{overlap}} = \mathcal{S}_{C} \cap \mathcal{S}_{\Omega}\).

The optimization task decouples per index \((i,j)\). For indices \((i,j)\) in the set \(\mathcal{S}_{\Omega}\) but not in \(\mathcal{S}_{C}\), the parameter \((C_2)_{ij}\) is forced to zero, so the solution is simply \((\widehat{\Omega}_1)_{ij} = \widehat{Z}_{ij}\). For indices \((i,j)\) in the set \(\mathcal{S}_{C}\) but not in \(\mathcal{S}_{\Omega}\), the parameter \((\widehat{\Omega}_1)_{ij}\) is forced to zero, so the solution is \((C_2)_{ij} = \widehat{Z}_{ij}\). For overlapping indices where both parameters are free, we face the underdetermined equation \((C_2)_{ij} + (\widehat{\Omega}_1)_{ij} = \widehat{Z}_{ij}\).

Among the parameter pairs representing this same orthogonal projection, we select the pair of minimum product Frobenius norm:
\[
\min_{c,\omega} (c^2 + \omega^2) \quad \text{subject to} \quad c + \omega = \widehat{Z}_{ij},
\]
which yields the symmetric solution \((C_2)_{ij} = (\widehat{\Omega}_1)_{ij} = \frac{1}{2}\widehat{Z}_{ij}\).

\begin{lemma}\label{lemma:overlap_indices}
Let \(P_1\) be the Perfect Shuffle permutation corresponding to the factorization \(N = k_1 \cdot b_1 = k_2 \cdot b_2\). Assuming \(0\)-based indexing for the rows and columns, the set of non-diagonal overlapping indices where both \(C_2\) and \(\widehat{\Omega}_1\) can be non-zero is explicitly given by:
\[ \mathcal{S}_{\mathrm{overlap}} = \left\{ (i, j) \in \{0, \dots, N-1\}^2 \;\Big|\; i \neq j, \;\; \lfloor i / b_2 \rfloor = \lfloor j / b_2 \rfloor, \; \text{and} \; i \equiv j \pmod{k_1} \right\}. \]
\end{lemma}

\begin{proof}
The first condition, \(\lfloor i / b_2 \rfloor = \lfloor j / b_2 \rfloor\), follows directly from the definition of \(C_2\), which is strictly block-diagonal with block size \(b_2\).

For the second condition, we apply Corollary \ref{corollary:similarity_PS}. Recall that \(\Omega_1\) is block-diagonal with \(k_1\) blocks of size \(b_1\). The similarity transformation \(\widehat{\Omega}_1 = P_1 \Omega_1 P_1^\top\) rearranges these dense blocks. According to Corollary \ref{corollary:similarity_PS}, the resulting matrix \(\widehat{\Omega}_1\) consists of blocks that are strictly diagonal. Geometrically, this implies that non-zero entries of \(\widehat{\Omega}_1\) occur only when the row and column indices correspond to the same offset within the shuffled block structure. For the Perfect Shuffle defined on \(N = k_1 \cdot b_1\), this structural invariant translates algebraically to the congruence \(i \equiv j \pmod{k_1}\).

Note that if \(k_1 \ge b_2\), the conditions \(|i - j| \le b_2 - 1\) (from the first constraint) and \(i \equiv j \pmod{k_1}\) (from the second) can only be satisfied simultaneously if \(i = j\). Since we require \(i \neq j\) for skew-symmetric matrices, \(\mathcal{S}_{\mathrm{overlap}}\) is empty, aligning with the disjoint support case. When \(k_1 < b_2\), the intersection is non-empty and consists precisely of indices bounded by the block size \(b_2\) and separated by a stride of~\(k_1\).
\end{proof}

\subsection{Riemannian Gradient via Automatic Differentiation}
\label{subsec:ad_implementation}
As shown, for example in \cite{NRO}, the Riemannian gradient on fixed-rank matrix or tensor-train manifolds can be computed efficiently using Automatic Differentiation (AD) without explicitly forming the full Euclidean gradient or manually implementing the projection operators. The core idea is to define an auxiliary function on the tangent space parameters and differentiate it. This leads to a black-box implementation compatible with automatic differentiation packages, thereby making Riemannian AD readily applicable to deep learning tasks.
Let us show that we can apply this approach efficiently in our context.

\paragraph{Case $k_1 \geq b_2$}
We formally derive the gradient of the auxiliary function \(g(\Omega_1, C_2)\). Let \(\nabla f(X)\) denote the Euclidean gradient of \(f\) at \(X\), defined such that the differential is \(\mathrm{d}f = \langle \nabla f(X), \mathrm{d}X \rangle\), where \(\langle A, B \rangle = \operatorname{tr}(A^\top B)\).

Consider the auxiliary function \(g(\Omega_1, C_2) = f(X + \xi(\Omega_1, C_2))\). The tangent vector parameterized by the skew-symmetric block-diagonal matrices is:
\begin{equation*}
\xi(\Omega_1, C_2) = B_2 C_2 P_1 B_1 + B_2 P_1 \Omega_1 B_1.
\end{equation*}
This decomposition forms an isometry from the parameter space to the tangent space, i.e., \(\|\xi\|_F^2 = \|\Omega_1\|_F^2 + \|C_2\|_F^2\). The cross-term vanishes because \(\operatorname{tr}(C_2^\top P_1 \Omega_1 P_1^\top) = 0\), as established in Section~\ref{subsec:proj_tangent}.

To find the partial gradients evaluated at zero, we take the differential of \(\xi\) with respect to the parameters \(C_2\) and \(\Omega_1\):
\begin{equation*}
\mathrm{d}\xi = B_2 \mathrm{d}C_2 P_1 B_1 + B_2 P_1 \mathrm{d}\Omega_1 B_1.
\end{equation*}
Substituting this into the differential of the objective function yields:
\begin{equation} \label{eq:differential_f}
\mathrm{d}g = \langle \nabla f(X), B_2 \mathrm{d}C_2 P_1 B_1 \rangle + \langle \nabla f(X), B_2 P_1 \mathrm{d}\Omega_1 B_1 \rangle.
\end{equation}
By the definition of the gradient, \(\mathrm{d}g = \langle \nabla_{C_2} g, \mathrm{d}C_2 \rangle + \langle \nabla_{\Omega_1} g, \mathrm{d}\Omega_1 \rangle\). We isolate \(\mathrm{d}C_2\) and \(\mathrm{d}\Omega_1\) using the adjoint property of the Frobenius inner product (\(\langle A, BCD \rangle = \langle B^\top A D^\top, C \rangle\)). For the first term in \eqref{eq:differential_f}:
\begin{equation*}
\langle \nabla f(X), B_2 \mathrm{d}C_2 P_1 B_1 \rangle = \langle B_2^\top \nabla f(X) B_1^\top P_1^\top, \mathrm{d}C_2 \rangle.
\end{equation*}
Thus, $\nabla_{C_2} g(0, 0) = B_2^\top \nabla f(X) B_1^\top P_1^\top$. For the second term in \eqref{eq:differential_f}:
\begin{equation*}
\langle \nabla f(X), B_2 P_1 \mathrm{d}\Omega_1 B_1 \rangle = \langle P_1^\top B_2^\top \nabla f(X) B_1^\top, \mathrm{d}\Omega_1 \rangle.
\end{equation*}
Thus, $\nabla_{\Omega_1} g(0, 0) = P_1^\top B_2^\top \nabla f(X) B_1^\top$.

Notice the direct relationship between these partial gradients and the auxiliary matrix $\widetilde{Z} := B_2^\top \nabla f(X) B_1^\top P_1^\top$ derived in the projection step. We immediately have:
\begin{equation*}
\nabla_{C_2} g(0,0) = \widetilde{Z}, \quad \text{and} \quad \nabla_{\Omega_1} g(0,0) =  P_1^\top \underbrace{\left( B_2^\top \nabla f(X) B_1^\top P_1^\top \right)}_{\widetilde{Z}} P_1 = P_1^\top \widetilde{Z} P_1.
\end{equation*}

To obtain the final Riemannian gradient directions on the parameter spaces, we apply the skew-symmetric block-diagonal projection. To project onto the tangent space, we need to find \(\mathcal{P}_{\text{bdiag}(b_2, k_2)}(\widehat{Z})\), which, in terms of the gradient of the auxiliary function, means \(\mathcal{P}_{\text{bdiag}(b_2, k_2)}(\operatorname{skew}(\nabla_{C_2} g (0,0)))\), and we need to find \( \Omega_1 = \mathcal{P}_{\text{bdiag}(b_1, k_1)} \left( P_1^\top \widehat{Z} P_1 \right) \). To rewrite the final expression in terms of the auxiliary function, we note that for an arbitrary matrix \(A\), there exist two equivalent forms for the skew-symmetric projection. The first variant computes the projection directly
\[
\Omega_{\text{manual}} = P_1^\top \operatorname{skew}(A) P_1 = P_1^\top \left( \frac{A - A^\top}{2} \right) P_1 = \frac{1}{2} \left( P_1^\top A P_1 - P_1^\top A^\top P_1 \right).
\]
The second variant first performs the similarity transformation \(G = P_1^\top A P_1\) and then applies the skew operator
\[
\Omega_{\text{AD}} = \operatorname{skew}(G) = \frac{G - G^\top}{2}.
\]
Substituting the expression for \(G\) confirms their equality
\[
\Omega_{\text{AD}} = \frac{1}{2} \left( (P_1^\top A P_1) - (P_1^\top A P_1)^\top \right) = \frac{1}{2} \left( P_1^\top A P_1 - P_1^\top A^\top P_1 \right) = \Omega_{\text{manual}}.
\]
Note also that
\begin{align*}
\mathcal{P}_{\text{bdiag}(b_i, k_i)} \left( \operatorname{skew}(A) \right) &= \mathcal{P}_{\text{bdiag}(b_i, k_i)} \left( \frac{A - A^T}{2} \right) \\
&= \frac{\mathcal{P}_{\text{bdiag}(b_i, k_i)}(A) - \mathcal{P}_{\text{bdiag}(b_i, k_i)}(A^T)}{2} \\
&=\operatorname{skew} \left( \mathcal{P}_{\text{bdiag}(b_i, k_i)}(A) \right).
\end{align*}
This confirms that computing the Euclidean gradient of the AD auxiliary function and simply projecting it block-wise recovers the exact Riemannian gradient.

\paragraph{Case \(k_1 < b_2\)}
We have found that all elements behave as in the case \( k_1 \geq b_2 \), except for those in the overlap between \( \widehat{\Omega}_1 \) and \( C_2 \). For these overlapping elements, a simple division by two is required.

Since the set of overlapping elements is predetermined (i.e., deterministic), we can reuse the AD trick as before --- without constructing full matrices, which ensures efficiency. The necessary step is to compute the indices of the overlapping elements and halve their values, thereby obtaining the correct projection onto the tangent space. To correct this efficiently without constructing full matrices, we use the precomputed mask \( \mathbf{M} \) from Lemma \ref{lemma:overlap_indices}. The Riemannian gradient updates are modified as follows:
\[
\dot{C}_2 \leftarrow \dot{C}_2 - \tfrac{1}{2} M \odot \dot{C}_2,
\qquad
\dot{\Omega}_1 \leftarrow \dot{\Omega}_1 - \tfrac{1}{2} (P_1^\top M P_1) \odot \dot{\Omega}_1.
\]
This operation is an element-wise multiplication by a sparse tensor, which is computationally inexpensive.

\subsection{Sharing} \label{subsec:sharing}
In many computational architectures, including deep learning models, it is highly desirable to impose a parameter-sharing constraint to reduce the number of parameters. In our case, this requires multiple blocks within the factors \(B_1\) or \(B_2\) to be identical. Let us formally define this constraint and derive the corresponding Riemannian projection.

Let \(u_1 \leq k_1\) and \(u_2 \leq k_2\) be the number of \textit{unique} blocks for the right and left factors, respectively. We define surjective mapping functions \(\tau_1: \{1, \dots, k_1\} \to \{1, \dots, u_1\}\) and \(\tau_2: \{1, \dots, k_2\} \to \{1, \dots, u_2\}\), which map the block index in the full matrix to the index of the unique parameter.

The manifold of \(\mathcal{GS}\)-matrices with sharing is a submanifold of the original space, where the factors \(B_1\) and \(B_2\) are constrained such that \(B_1^{(i)} = B_1^{(j)}\) if \(\tau_1(i) = \tau_1(j)\) (and similarly for \(B_2\)). Consequently, the tangent vectors must satisfy the same structural constraints.
The tangent parameters \(\Omega_1\) and \(C_2\) are now constructed from sets of unique parameters \(\{\widetilde{\Omega}^{(p)}\}_{p=1}^{u_1}\) and \(\{\widetilde{C}^{(q)}\}_{q=1}^{u_2}\)
\begin{equation*}
\Omega_1^{(i)} = \widetilde{\Omega}^{(\tau_1(i))}, \quad C_2^{(j)} = \widetilde{C}^{(\tau_2(j))}.
\end{equation*}

\begin{remark}
In our sharing method, blocks may be identical across left and right matrices, but a single block cannot be shared between both \(B_1\) and $B_2$.
\end{remark}
Smoothness of the manifold can be established analogously to the case without sharing. Retraction also does not require significant modifications: it acts blockwise, and its proof is analogous to the case without the sharing mechanism.

\subsubsection{Projection with Sharing Constraints}
We project a Euclidean gradient \(Z\) onto the tangent space with parameter sharing, optimizing over unique parameters \(\{\widetilde{C}^{(q)}\}\) and \(\{\widetilde{\Omega}^{(p)}\}\)
\begin{equation*}
\min_{\substack{\{\widetilde{C}\}, \{\widetilde{\Omega}\}}} \left\| \widehat{Z} - \left(C_{\text{full}}(\{\widetilde{C}\}) + \widehat{\Omega}_{\text{full}}(\{\widetilde{\Omega}\}) \right) \right\|_F^2,
\end{equation*}
where \(\widehat{Z} := \operatorname{skew}(B_2^\top Z B_1^\top P_1^\top)\) was derived in the previous section.
Assuming the dimension constraint \(k_1 \geq b_2\) holds, the subspaces of \(C\) and \(\widehat{\Omega}\) remain orthogonal (disjoint support of non-zero elements), allowing us to decouple the minimization problem into
\begin{equation*}
\label{eq:shared_obj}
\sum_{j=1}^{k_2} \big\| \widehat{Z}_{jj} - C_{\text{full}}^{(j)} \big\|_F^2 + \sum_{i=1}^{k_1} \big\| G_{ii} - \Omega_{\text{full}}^{(i)} \big\|_F^2,
\end{equation*}
where \(G = P_1^\top \widehat{Z} P_1\), and \(\widehat{Z}_{jj}\) denote the diagonal blocks of \(\widehat{Z}\). Substituting the sharing constraints, the first term (for \(C\)) becomes
\begin{equation*}
\sum_{j=1}^{k_2} \big\| \widehat{Z}_{jj} - \widetilde{C}^{(\tau_2(j))} \big\|_F^2 = \sum_{q=1}^{u_2} \sum_{j \in \mathcal{J}_q} \big\| \widehat{Z}_{jj} - \widetilde{C}^{(q)} \big\|_F^2,
\end{equation*}
where \(\mathcal{J}_q = \{j \mid \tau_2(j) = q\}\) is the set of block indices sharing the \(q\)-th parameter.
Minimizing this quadratic form with respect to the unique parameter \(\widetilde{C}^{(q)}\) yields the arithmetic mean of the corresponding blocks of the gradient \(\widehat{Z}\)
\begin{equation*}
\widetilde{C}^{(q)} = \frac{1}{|\mathcal{J}_q|} \sum_{j \in \mathcal{J}_q} \widehat{Z}_{jj}.
\end{equation*}
Similarly, for the unique parameters of \(\Omega\)
\begin{equation*}
\widetilde{\Omega}^{(p)} = \frac{1}{|\mathcal{I}_p|} \sum_{i \in \mathcal{I}_p} G_{ii}, \quad \text{where } \mathcal{I}_p = \{i \mid \tau_1(i) = p\}.
\end{equation*}

\subsubsection{AD Implementation for Shared Parameters}
Now let us adapt AD for shared parameters. We treat the optimization with shared parameters as a composition of two functions: first, a sharing map that sends the unique parameters \(\widetilde{\mathbf{C}}, \widetilde{\mathbf{\Omega}}\) to the full block‑diagonal matrices \(C_{\text{full}}, \Omega_{\text{full}}\); second, the objective \(F(C_{\text{full}}, \Omega_{\text{full}}) = f(X + \xi(C_{\text{full}}, \Omega_{\text{full}}))\). The auxiliary function is then defined as 
\[g(\widetilde{\mathbf{C}}, \widetilde{\mathbf{\Omega}}) = F(C_{\text{full}}(\widetilde{\mathbf{C}}), \Omega_{\text{full}}(\widetilde{\mathbf{\Omega}})).\]
Applying the multivariable chain rule to find the gradient for a unique  \(\widetilde{C}^{(q)}\):
\begin{equation*}
\nabla_{\widetilde{C}^{(q)}} g = \sum_{j=1}^{k_2} \frac{\partial F}{\partial C_{\text{full}}^{(j)}} \cdot \frac{\partial C_{\text{full}}^{(j)}}{\partial \widetilde{C}^{(q)}}.
\end{equation*}
The partial derivative \(\frac{\partial C_{\text{full}}^{(j)}}{\partial \widetilde{C}^{(q)}}\) is the identity matrix \(I\) if block \(j\) is a copy of parameter \(q\) (i.e., \(j \in \mathcal{J}_q\)), and zero otherwise. The term \(\frac{\partial F}{\partial C_{\text{full}}^{(j)}}\) denotes the gradient of the objective with respect to the \(j\)-th block of the full matrix, treated as an independent variable. In Section \ref{subsec:proj_tangent}, we proved that for independent blocks, this gradient is the diagonal block of the matrix \(\widetilde{Z} = B_2^\top (\nabla_X f) B_1^\top P_1^\top\). Consequently, \(\frac{\partial F}{\partial C_{\text{full}}^{(j)}} = \widetilde{Z}_{jj}\).
Substituting these into the sum
\begin{equation}
\nabla_{\widetilde{C}^{(q)}} g = \sum_{j \in \mathcal{J}_q} \widetilde{Z}_{jj}.
\end{equation}

Similarly, for a unique parameter \(\widetilde{\Omega}^{(p)}\), we apply the chain rule. Recall that the gradient of the objective with respect to a block of the inner factor \(\Omega_{\text{full}}^{(i)}\) is given by the diagonal block of the rotated matrix \(G = P_1^\top \widetilde{Z} P_1\). Specifically, \(\frac{\partial F}{\partial \Omega_{\text{full}}^{(i)}} = G_{ii}\). Summing over all instances sharing the parameter \(p\):
\begin{equation*}
\nabla_{\widetilde{\Omega}^{(p)}} g = \sum_{i \in \mathcal{I}_p} \frac{\partial F}{\partial \Omega_{\text{full}}^{(i)}} = \sum_{i \in \mathcal{I}_p} G_{ii}.
\end{equation*}

We now reconcile the AD result (sum of gradients) with the analytic projection (average of gradients). 
The AD engine computes the dense gradient sum \(\sum \widetilde{Z}_{jj}\). To obtain the Riemannian gradient \(\dot{\widetilde{C}}^{(q)}\), we apply the skew-symmetric projection (which is a linear operator)
\[
\dot{\widetilde{C}}^{(q)} = \operatorname{skew}\left( \nabla_{\widetilde{C}^{(q)}} g \right) = \operatorname{skew}\left( \sum_{j \in \mathcal{J}_q} \widetilde{Z}_{jj} \right) = \sum_{j \in \mathcal{J}_q} \operatorname{skew}(\widetilde{Z}_{jj}).
\]
Since \(\widehat{Z}_{jj} = \operatorname{skew}(\widetilde{Z}_{jj})\), we have \(\dot{\widetilde{C}}^{(q)}_{\text{AD}} = \sum_{j \in \mathcal{J}_q} \widehat{Z}_{jj}.\) Comparing this with the analytic solution, which is the arithmetic mean \(\widetilde{C}^{(q)}_{\text{analytic}} = \frac{1}{|\mathcal{J}_q|} \sum_{j \in \mathcal{J}_q} \widehat{Z}_{jj}\), we obtain \(\dot{\widetilde{C}}^{(q)}_{\text{AD}} = |\mathcal{J}_q| \cdot \widetilde{C}^{(q)}_{\text{analytic}}. \) Similarly, we obtain the following equation \(\dot{\widetilde{\Omega}}^{(p)}_{\text{AD}} = |\mathcal{I}_p| \cdot \widetilde{\Omega}^{(p)}_{\text{analytic}}.\) Therefore, with parameter sharing, AD returns sums over all tied copies. The Riemannian gradient is obtained by dividing each shared block gradient by its multiplicity.

\begin{remark}
With parameter sharing and overlapping supports (\(k_1 < b_2\)), AD requires a minor correction. It sums gradients over shared instances; for indices in the overlap set \(\mathcal{S}_{\mathrm{overlap}}\), projection halves the gradient.

However, for arbitrary sharing patterns the element-wise $1/2$ mask is no longer providing minimal norm solution. In such cases, extracting the exact overlap components requires solving a small, precomputable least-squares problem
\[
\min_{\{\widetilde{C}\}, \{\widetilde{\Omega}\}} \sum_{(i,j) \in \mathcal{S}_{\text{overlap}}} \left\| \widehat{Z}_{ij} - \widetilde{C}^{(\tau_2(i))} - \widetilde{\Omega}^{(\tau_1(j))} \right\|_F^2.
\]
\end{remark}

\subsection{Retraction} \label{subsec:retraction} 
In this subsection, we show that the block-wise retraction functions as a valid retraction onto the manifold \(\mathcal{A}_2^{\mathrm{orth}}\).

Let \( X = B_2 P_1 B_1 \in \mathcal{A}_2^{\mathrm{orth}} \), and let \(\xi \in T_X \mathcal{A}_2^{\mathrm{orth}}\) be a tangent vector. By taking derivatives of the parameterized mapping \(\Phi(B_2, B_1) = B_2 P_1 B_1\), any tangent vector can be expressed as \(\xi = \eta_2 P_1 B_1 + B_2 P_1 \eta_1\), where \(\eta_i \in T_{B_i} \mathcal{B}_i\) are tangent vectors to the block-diagonal orthogonal factors. 

Let \(\mathrm{Retr}_{B_i}(\eta_i)\) be a standard retraction onto \(\mathcal{B}_i\), applied block-wise. We define the retraction \(\gamma_X(\xi, t)\) onto the manifold \(\mathcal{A}_2^{\mathrm{orth}}\) along the tangent direction $\xi$ as follows:
\begin{equation*}
\begin{split}
\gamma_X(\xi, t) &:= \gamma_X(\xi(\eta_1, \eta_2), t) = \Phi\left(\mathrm{Retr}_{B_2}(t\eta_2),\,\mathrm{Retr}_{B_1}(t\eta_1)\right) \\
&= \left(B_2 + t\eta_2 + \mathcal{O}(t^2)\right) P_1 \left(B_1 + t\eta_1 + \mathcal{O}(t^2)\right) \\
&= B_2 P_1 B_1 + t (\eta_2 P_1 B_1 + B_2 P_1 \eta_1) + \mathcal{O}(t^2) \\
&= X + t\xi + \mathcal{O}(t^2).
\end{split}
\end{equation*}
By definition, \(\gamma_X(\xi, 0) = X\) and \(\frac{\mathrm{d}}{\mathrm{d}t}\gamma_X(\xi, t) \big|_{t=0} = \xi\). To establish that this is a retraction, it remains to verify that \(\gamma_X(\xi,t)\) is a well-defined and smooth function of \(\xi\). For this, it remains to ensure that the resulting matrix is independent of the choice of factor tangent vectors representing \(\xi\).

When \(k_1 \ge b_2\), the pair of vectors
\((\eta_2,\eta_1)\) representing a given \(\xi\) is unique (and smooth). This follows from the preservation of vector norms (see Appendix~\ref{appendix:geodesics}), which ensures injectivity; surjectivity then follows since the tangent spaces have the same dimension (see \cite[Appendix B]{OF}). Hence, the constructed map serves as a valid retraction \cite[Definition 4.1.1]{Absil2008}.  Alternatively, one can interpret this as follows: since the blockwise retractions provide a first‑order (or higher, if applicable) approximation to the corresponding blockwise exponential maps (geodesics), and \(\Phi\) is a local isometry (see Appendix~\ref{appendix:geodesics}) and smooth, the composition  
\(\Phi\!\left(\mathrm{Retr}_{B_2}(t\eta_2),\,\mathrm{Retr}_{B_1}(t\eta_1)\right)\)  
approximates the exponential map (geodesics) on \(\mathcal{A}_2^{\mathrm{orth}}\). Consequently, \(\Phi\!\left(\mathrm{Retr}_{B_2}(t\eta_2),\,\mathrm{Retr}_{B_1}(t\eta_1)\right)\) is indeed a retraction.

\begin{remark} \label{rem:orthogonal_equivariance}
Note that since \(\Omega_1 = B_1 C_1 B_1^\top\) and the retractions we use (exponential, Cayley, polar) are analytic and satisfy \(f(B^\top \Omega B) = B^\top f(\Omega) B\) for any orthogonal \(B\), the update  
\(B_1^{\text{new}} = \operatorname{Retr}(\Omega_1) \cdot B_1^{\text{old}}\) is exactly equivalent to the standard right‑translation update \(B_1^{\text{new}} = B_1^{\text{old}} \cdot \operatorname{Retr}(C_1)\). We work directly with \(\Omega_1\) to avoid the extra matrix multiplications needed to recover \(C_1\) at every step.    
\end{remark}

For the case \(k_1 < b_2\), the construction suffers from an ambiguity. The significance of our prior choice of the minimum-norm solution becomes apparent at this stage: it ensures uniqueness of the resulting map \( \Phi\!\left(\operatorname{Retr}_{B_2}(t\eta_2),\,\operatorname{Retr}_{B_1}(t\eta_1)\right)\). Indeed, recall that $\xi = B_2 C_2 P_1 B_1 + B_2 P_1 B_1 C_1 = \Omega_2 B_2 P_1 B_1 + B_2 P_1 B_1 C_1 = \Omega_2 X + X C_1.$ Now define $L_X : \mathfrak{so}_{\text{block}}(b_2, k_2) \times \mathfrak{so}_{\text{block}}(b_1, k_1) \longrightarrow T_X \mathcal{A}_2^{\text{orth}},$ $L_X(\Omega_2, C_1) = \Omega_2 X + X C_1$ and observe that the minimum-norm problem
\[
\min_{L_X(\Omega_2,C_1)=\xi} \left(
\|\Omega_2\|_F^2+\|C_1\|_F^2 \right)
\]
has a nonempty feasible set and a strictly convex objective, so its minimizer \((\Omega_2^\ast, C_1^\ast)=L_X^\dagger \xi\) is uniquely determined by \((X,\xi)\). For the blockwise retractions considered in remark \ref{rem:orthogonal_equivariance}, orthogonal equivariance gives
\[
\operatorname{Retr}_{B_2}(t\Omega_2^\ast B_2)
=
\operatorname{Retr}_{I}(t\Omega_2^\ast)B_2,
\qquad
\operatorname{Retr}_{B_1}(tB_1C_1^\ast)
=
B_1\operatorname{Retr}_{I}(tC_1^\ast).
\]
Therefore, for any factorization \(X=B_2P_1B_1\),
\[
\Phi\!\left(\operatorname{Retr}_{B_2}(t\eta_2^\ast),
\operatorname{Retr}_{B_1}(t\eta_1^\ast)\right)
=
\operatorname{Retr}_{I}(t\Omega_2^\ast)\,
X\, \operatorname{Retr}_{I}(tC_1^\ast).
\]
Moreover, \(L_X\) is surjective onto 
\(T_X\mathcal{A}_2^{\mathrm{orth}}\).
Hence \(L_X^\dagger = L_X^\ast(L_X L_X^\ast)^{-1}\) depends smoothly on \(X\), because \(L_X\) does and \(L_X L_X^\ast\) remains invertible, and therefore \(\gamma_X(\xi,t)\) is smooth in \((X,\xi,t)\).
The resulting matrix is also independent of the particular factorization of \(X\) and hence the retraction
\[
\gamma_X(\xi,t)
=
\Phi\!\left(\mathrm{Retr}_{B_2}(t\eta_2^\ast),\,\mathrm{Retr}_{B_1}(t\eta_1^\ast)\right)
\]
is well-defined as a function of \(\xi\) alone (given fixed $X$).

\subsection{Complexity Analysis}

We analyze the computational and storage complexity of the proposed Riemannian optimization framework. Let $\mathcal{C}_f$ denote the arithmetic complexity of evaluating the objective function $f(A)$ using the block-diagonal factors $B_1$ and $B_2$. For each factor $i \in \{1, 2\}$, let $k_i$ be the number of blocks of size $b_i \times b_i$ (such that $k_i b_i = N$), and let $u_i \le k_i$ be the number of unique blocks under the parameter sharing constraint.

\begin{proposition}\label{prop:time_complexity}
The total arithmetic complexity for a single Riemannian optimization step is
\[ \mathcal{C}_{\mathrm{time}} = \mathcal{O}\left( \mathcal{C}_f + N(b_1 + b_2) + u_1 b_1^3 + u_2 b_2^3 \right). \]
\end{proposition}

\begin{proof}
A single optimization step consists of three main stages: computing the Euclidean gradients, projecting them onto the tangent space, and applying the retraction.

Note that both computing the Euclidean gradients and using automatic differentiation require \(\mathcal{O}(\mathcal{C}_f)\) operations (see \cite[Section 18.9]{Tyrtyshnikov1997}.) Second, projecting onto the tangent space further requires extracting the block‑diagonal part, taking the skew‑symmetric components, and applying the overlap correction mask from the AD results (permutation matrix multiplications may also be involved, but they cost \(\mathcal{O}(N)\)). Since these are element-wise operations bounded by the number of non-zero entries in the block-diagonal matrices, the cost is proportional to $k_1 b_1^2 + k_2 b_2^2 = N(b_1 + b_2)$. Thus, the tangent space projection requires $\mathcal{O}(N(b_1 + b_2))$ operations.

Finally, the optimizer must map the tangent vectors back to the manifold. We compute block-wise retraction for the $u_i$ unique blocks. Retracting a block of size $b_i \times b_i$ costs $\mathcal{O}(b_i^3)$. Hence, the retraction step costs $\mathcal{O}(u_1 b_1^3 + u_2 b_2^3)$. Summing the gradient computation cost and the manifold overhead gives the stated bound.
\end{proof}

\begin{proposition} \label{prop:space_complexity}
Let $\mathcal{S}_f$ be the complexity required to evaluate $f(A)$ (including intermediate activations). The total memory footprint of the optimizer is $\mathcal{C}_{\mathrm{space}} = \mathcal{O}(\mathcal{S}_f + u_1 b_1^2 + u_2 b_2^2)$. The initialization of the overlap mask requires $\mathcal{O}(N \cdot \max(b_1, b_2))$ integer operations and incurs a negligible $\mathcal{O}(b_2^2)$ storage overhead.
\end{proposition}

\begin{proof}
The manifold degrees of freedom require storing only the $u_i$ unique blocks of size $b_i \times b_i$. This requires $\mathcal{O}(u_1 b_1^2 + u_2 b_2^2)$ storage to the memory footprint. Furthermore, the tangent space projection mask $\mathbf{M}$ (derived in Lemma \ref{lemma:overlap_indices}) is evaluated exactly once prior to training. Constructing it involves verifying the congruence $i \equiv j \pmod{k_1}$ strictly within the local $b_2 \times b_2$ block neighborhoods. By evaluating this condition for the $N$ diagonal elements within a bounding radius of $\max(b_1, b_2)$, the mask generation is limited to $\mathcal{O}(N \cdot \max(b_1, b_2))$ fast integer modulo operations. Because the mask repeats across blocks, the resulting constant Boolean tensor is stored compactly with a size bounded by $\mathcal{O}(b_2^2)$, adding almost zero persistent memory overhead during the iterative training process.
\end{proof}

\subsection{Fixed Boundary Permutations}
In practical applications, the \(\mathcal{GS}\) structure often includes fixed permutation matrices at the input and output boundaries
\begin{equation}
\label{eq:boundary_perm}
A = P_{\text{out}} \left( B_m P_{m-1} \dots B_1 \right) P_{\text{in}},
\end{equation}
where \(P_{\text{in}}\) and \(P_{\text{out}}\) are constant (non-trainable) permutation matrices. For instance, one may initialize \(P_{\text{in}} = I\), all block-diagonal factors \(B_i = I\), and set \(P_{\text{out}} = P_1^\top\) (the inner permutation). With this choice the overall transformation \(A\) starts as the identity, providing a natural initialization for neural network training. Notice that this also enables us to learn the Kronecker product \( U \otimes V \), where \( U \) and \( V \) are fully shared across the corresponding block‑diagonal factor blocks, using Lemma \ref{similarity_PS}.

Since permutations are orthogonal matrices, the mapping \(\Psi(X) = P_{\text{out}} X P_{\text{in}}\) defines an isometry of the ambient Euclidean space with respect to the Frobenius metric. Consequently, the Riemannian optimization framework requires only minimal adjustments:
\begin{enumerate}
    \item It is straightforward to show that the retraction remains blockwise.
    \item Computational Cost: Applying a fixed permutation only requires reindexing the entries and has complexity linear in the size of the input tensor.
    \item Let \(G_{\text{full}} = \nabla_A f\) be the Euclidean gradient of the objective with respect to the full matrix \(A\). By definition of the Frobenius inner product, the gradient \(G_{\text{core}} = \nabla_X f\) satisfies:
    \[ \langle G_{\text{full}}, \mathrm{d}A \rangle = \operatorname{tr}\left( G_{\text{full}}^\top P_{\text{out}} \mathrm{d}X P_{\text{in}} \right). \]
    Using the cyclic property of the trace and the orthogonality of permutation matrices, we obtain:
    \[ \operatorname{tr}\left( P_{\text{in}} G_{\text{full}}^\top P_{\text{out}} \mathrm{d}X \right) = \langle P_{\text{out}}^\top G_{\text{full}} P_{\text{in}}^\top, \mathrm{d}X \rangle. \]
    Therefore, the exact Euclidean gradient with respect to the core matrix \(X\) is simply:
    \[ G_{\text{core}} = P_{\text{out}}^\top G_{\text{full}} P_{\text{in}}^\top. \]
    Computationally, this implies that we only need to apply inverse permutations (memory reordering, \(\mathcal{O}(MN)\) complexity) to the incoming gradient before passing it to our Riemannian optimization logic.
    \item By the same reasoning, the Riemannian gradient should be left‑multiplied by \(P_{\text{out}}^\top\) and right‑multiplied by \(P_{\text{in}}^\top\).
\end{enumerate}

\subsection{Vector transport}
Many optimizers (e.g., SGD with momentum) require accumulating gradient information over time, which necessitates transporting vectors from the tangent space at the current point to the tangent space at the next point. For embedded submanifolds of a Euclidean space, the orthogonal projection of a tangent vector from the tangent space of one point onto the tangent space of another serves as a natural and rigorously justified vector transport (see \cite[Section 8.1.3]{Absil2008}).

We now need to compute the orthogonal projection of a tangent vector efficiently, without forming the full matrices. In what follows, we rely on the following lemma, which is a direct consequence of block matrix multiplication rules; we omit its proof for brevity.

\begin{lemma} \label{lemma:block_diag_proj}
If a matrix \(A\) is strictly block-diagonal with \(k_i\) blocks of size \(b_i \times b_i\), then for any compatible matrix \(B\),
\[
\mathcal{P}_{\mathrm{bdiag}(b_i, k_i)}(A \cdot B) = A \cdot \mathcal{P}_{\mathrm{bdiag}(b_i, k_i)}(B),
\]
and analogously
\[
\mathcal{P}_{\mathrm{bdiag}(b_i, k_i)}(B \cdot A) = \mathcal{P}_{\mathrm{bdiag}(b_i, k_i)}(B) \cdot A.
\] 
\end{lemma}

According to Section \ref{subsec:proj_tangent}, to project any ambient vector \( Z \) onto \( T_{X_{k+1}} \mathcal{A}_2^{\mathrm{orth}} \), we first compute the auxiliary matrix \( \widetilde{Z} = \widetilde{B}_2^\top Z \widetilde{B}_1^\top P_1^\top \), extract the corresponding block diagonals from it (and from its permuted counterpart), and then apply the skew-symmetric projection along with the overlap mask (if needed).

For a tangent momentum vector \( Z = B_2 C_2 P_1 B_1 + B_2 P_1 \Omega_1 B_1 \in T_{X_{k}} \mathcal{A}_2^{\mathrm{orth}} \), we substitute it into the projection formula:
\[
\widetilde{Z} = \widetilde{B}_2^\top (B_2 C_2 P_1 B_1 + B_2 P_1 \Omega_1 B_1) \widetilde{B}_1^\top P_1^\top.
\]
By introducing the block-diagonal transition matrices \( W_2 = \widetilde{B}_2^\top B_2 \) and \( W_1 = B_1 \widetilde{B}_1^\top \), the expression simplifies to:
\[
\widetilde{Z} = W_2 C_2 P_1 W_1 P_1^\top + W_2 P_1 \Omega_1 W_1 P_1^\top.
\]
Notice that the parameters of the transported vector, denoted as \(\widetilde{C}_2\) and \(\widetilde{\Omega}_1\), are obtained by projecting \(\widetilde{Z}\) and \(P_1^\top \widetilde{Z} P_1\), respectively. Since the transition matrices \(W_i\) and the generators \(C_2, \Omega_1\) are strictly block-diagonal, we can apply Lemma \ref{lemma:block_diag_proj} to factor them outside the projection operators.

To simplify the notation, we introduce two structure-preserving index permutation operators:
\[
\Pi_{1 \to 2}(M) := \mathcal{P}_{\mathrm{bdiag}(b_2, k_2)}(P_1 M P_1^\top), \quad \Pi_{2 \to 1}(K) := \mathcal{P}_{\mathrm{bdiag}(b_1, k_1)}(P_1^\top K P_1).
\]
Using Lemma \ref{lemma:block_diag_proj} to extract the block-diagonal components \(\widehat{C}_2\) and \(\widehat{\Omega}_1\), the projection perfectly decouples into block-wise multiplications:
\begin{align*}
\widehat{C}_2 &= \mathcal{P}_{\mathrm{bdiag}(b_2, k_2)}(\widetilde{Z}) = W_2 C_2 \Pi_{1 \to 2}(W_1) + W_2 \Pi_{1 \to 2}(\Omega_1 W_1), \\
\widehat{\Omega}_1 &= \mathcal{P}_{\mathrm{bdiag}(b_1, k_1)}(P_1^\top \widetilde{Z} P_1) = \Pi_{2 \to 1}(W_2 C_2) W_1 + \Pi_{2 \to 1}(W_2) \Omega_1 W_1.
\end{align*}
Note that the projection involves block-diagonal matrix multiplication, which is efficiently performed blockwise. For \(\Pi_{1 \to 2}(M)\) and \(\Pi_{2 \to 1}(K)\), under the assumption \(k_1 \geq b_2\) we may simply apply Corollary \ref{corollary:similarity_PS}, which explicitly reveals that the resulting block-diagonal matrices are strictly diagonal. They are constructed exclusively from the permuted diagonal entries of the original matrices:
\[
\Pi_{1 \to 2}(M) = P_1 \operatorname{diag}(M) P_1^\top \quad \text{and} \quad \Pi_{2 \to 1}(K) = P_1^\top \operatorname{diag}(K) P_1,
\]
where \(\operatorname{diag}(\cdot)\) is the operator that sets all off-diagonal elements to zero. Finally, the valid transported tangent vector parameters are recovered by enforcing skew-symmetry: \(\widetilde{C}_2 = \operatorname{skew}(\widehat{C}_2)\) and \(\widetilde{\Omega}_1 = \operatorname{skew}(\widehat{\Omega}_1)\). In the overlap scenario, an extra mask would be required.

Note that after a retraction step we have \(\widetilde{B}_1 = S_\Omega B_1\) and 
\(\widetilde{B}_2 = B_2 S_C\).  The transition matrices introduced in the projection
therefore analytically simplify to \(W_1 = B_1 \widetilde{B}_1^\top = B_1 (S_\Omega B_1)^\top = S_\Omega^\top\) and \(W_2 = \widetilde{B}_2^\top B_2 = (B_2 S_C)^\top B_2 = S_C^\top\).
The operators \(\Pi_{1\to2}\) and \(\Pi_{2\to1}\) act only on block‑diagonal matrices,
hence the whole transport reduces to a sequence of block‑diagonal multiplications. Consequently, the arithmetic cost is
\(\mathcal{O}(u_1 b_1^3 + u_2 b_2^3)\),
while the temporary memory overhead remains \(\mathcal{O}(u_1 b_1^2 + u_2 b_2^2)\) numbers.

The Riemannian framework for the orthogonal case generalizes to several other compact Lie groups, requiring only minor adjustments; we provide a brief discussion in Appendix~\ref{appendix:compact_groups}.

\section{Experiments} \label{sec:experiments}

To validate our theoretical results, we conducted several numerical experiments within our proposed training framework.
For convenience, we implemented a PyTorch library that supports Riemannian gradient descent with the following features: various retractions (exponential, Cayley, polar) --- all of which are valid retractions (see \cite[Example 4.1.2]{Absil2008}); overlap correction (if needed); as well as orthogonal blocks and complex‑valued cases. The library is available at \url{https://github.com/alialiev78/torchgs}.

\subsection{Synthetic tasks} \label{subsec:synthetic_experiments}

First, we investigated model learning on synthetic tasks. We solve a \(\mathcal{GS}\) Procrustes problem: find the nearest \(\mathcal{GS}\)-orthogonal matrix to a given target, considering both the case where the target itself is a \(\mathcal{GS}\) matrix and the case where it is an arbitrary (non-\( \mathcal{GS} \)) matrix. The baseline uses a Cayley parametrization with Euclidean optimization (Adam/SGD), while our method uses Riemannian gradient descent on the manifold. For a fair comparison, each block in the baseline is constructed via the Cayley transform to ensure orthogonality, and the same grid is applied uniformly across all methods.

The results for the \(\mathcal{GS}\)-target case are presented in Figure~\ref{fig:gs_target}, and for the arbitrary target case in Figure~\ref{fig:arbitrary_target}. Our method achieves comparable or better accuracy than the Cayley transform baseline.

\begin{figure}[t]
    \centering
    \includegraphics[width=0.8\textwidth]{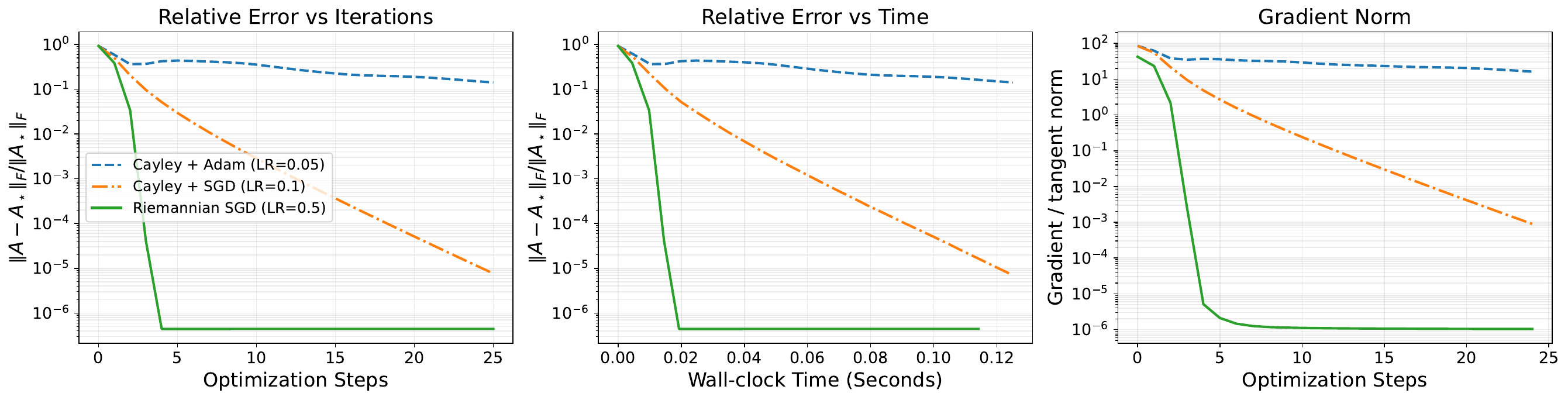}
    \caption{Performance on the $\mathcal{GS}$-target Procrustes problem (non‑overlap, $b_i = k_i = 32$). 
    Left: loss versus iterations; 
    center: loss versus time (seconds); 
    right: comparison of the Euclidean gradient norm (Cayley baseline) and the norm of the projected Riemannian gradient (our RGD).}
    \label{fig:gs_target}
\end{figure}

\begin{figure}[t]
    \centering
    \includegraphics[width=0.8\textwidth]{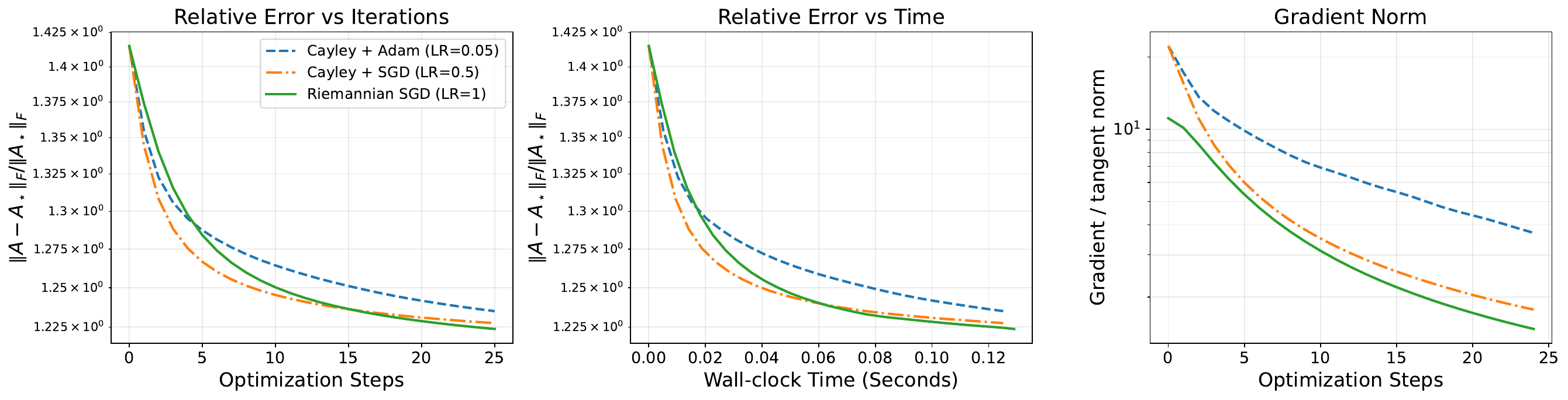}
    \caption{Performance on the arbitrary (non‑$\mathcal{GS}$) target (non‑overlap, $b_i = k_i = 32$). 
    Left: loss versus iterations; 
    center: loss versus time (seconds); 
    right: comparison of the Euclidean gradient norm (Cayley baseline) and the norm of the projected Riemannian gradient (our RGD).}
    \label{fig:arbitrary_target}
\end{figure}

We also conducted experiments with parameter sharing, overlap, and the complex case. The results and conclusions are qualitatively similar; for brevity, we omit them here, but they are available in the file \texttt{gs\_testing.ipynb} in our GitHub repository.

\subsection{Performance in Large-Scale Fine-Tuning} \label{subsec:bert_experiments}
To evaluate our manifold framework on large-scale tasks, we integrate $\mathcal{GS}$-orthogonal matrices into pre-trained architectures. For a frozen pre‑trained linear layer \(W\), we replace it with \(W_{\text{injected}} = A W\), where \(A = P_1^\top B_2 P_1 B_1\) is our $\mathcal{GS}$‑adapter. Additionally, to allow flexible scaling, we multiply the output by a learnable diagonal matrix \(S\): \(W_{\text{new}} = S W_{\text{injected}}\). To isolate the effect of \( S \), we use the same AdamW optimizer, learning rate (fixed $10^{-4}$ in all cases), schedule, and weight decay for \( S \) in both Euclidean and Riemannian setups. The \(\mathcal{GS}\) parts are trained identically, except that we add regularization to the Cayley transform in the Euclidean case, which is not needed in the Riemannian one since the manifold constraint already enforces stability.

We fine-tune RoBERTa on the GLUE benchmark using small block sizes \(b_i = 8\) (and \(k_i = 96\)) and compare against Euclidean baselines. Our overall setup follows the approach of \cite{Liu2024,GS}; further implementation details are available in those references or in our code repository. We used polar retraction (employing five Newton-Schulz iterations instead of exact computation). For SGD, we search over the learning rate grid \(\{1\times 10^{-3}, 5\times 10^{-3}, 1\times 10^{-2}, 5\times 10^{-2}, 8\times 10^{-2}, 1\times 10^{-1}\}\) with momentum \(\beta = 0.9\). For the AdamW optimizer, we search over \(\{1\times 10^{-4}, 5\times 10^{-4}, 8\times 10^{-4}, 1\times 10^{-3}, 2\times 10^{-3}, 5\times 10^{-3}\}\), with \(\beta = 0.95\) and keep the default PyTorch hyperparameters \(\beta_1 = 0.9, \beta_2 = 0.999, \epsilon = 10^{-8}\). The results are reported in Table~\ref{tab:glue_results}. Riemannian SGD outperforms Euclidean SGD in both performance and stability, though it may not quite reach Adam-level results.

\begin{table}[htbp]
\centering
\small
\caption{Results on the GLUE benchmark with RoBERTa-base. Pearson correlation is reported for STS-B, Matthew's correlation for CoLA, and accuracy for the other tasks. \# Params denotes the number of trainable parameters. Best results are in bold.}
\label{tab:glue_results}
\resizebox{\linewidth}{!}{
\begin{tabular}{lcccccccccc}
\toprule
\textbf{Method} & \textbf{\# Params} & \textbf{MNLI} & \textbf{SST-2} & \textbf{CoLA} & \textbf{QQP} & \textbf{QNLI} & \textbf{RTE} & \textbf{MRPC} & \textbf{STS-B} & \textbf{ALL} \\
\midrule
GSOFT$_{b=8}$ (SGD with momentum) & 1.38M & 86.42 & 93.46 & 51.76 & 86.01 & 87.31 & \textbf{81.23} & \underline{89.71} & 89.09 & 83.12 \\
GSOFT$_{b=8}$ (AdamW)              & 1.38M & \textbf{86.77} & \textbf{94.72} & \textbf{63.83} & \textbf{90.10} & \textbf{92.35} & \underline{79.42} & \textbf{90.44} & \textbf{90.70} & \textbf{86.04} \\
Riemannian SGD$_{b=8}$ (with momentum) & 1.38M & \underline{86.58} & \underline{94.38} & \underline{61.67} & \underline{89.02} & \underline{91.98} & 75.45 & 89.22 & \underline{90.52} & \underline{84.85} \\
\bottomrule
\end{tabular}
}
\end{table}

Note that the runtime gap between Euclidean and Riemannian optimization is further reduced here: since vector‑Jacobian products through frozen parameters dominate the computation, the overhead of manifold-constrained updates becomes small. In our experiments, across all tasks, Riemannian methods show a moderate average overhead of 5-7\% relative to Euclidean baselines.

For the smaller RoBERTa tasks, we conducted additional experiments using Riemannian SGD, exploring a broader range of configurations and block sizes. The detailed experimental settings and corresponding results are reported in Table~\ref{tab:rsgd_light_tasks}.

\begin{table}[htbp]
\centering
\scriptsize
\caption{All configurations use \(N=b_1k_1=b_2k_2=768\). 
\# Unique params denotes the number of independent trainable parameters after parameter sharing, whereas \# Full GS params denotes the number of parameters in the corresponding $\mathcal{GS}$ representation before sharing. RoBERTa-base contains approximately 125M parameters in total. Best results per column are in \textbf{bold}, second best are \underline{underlined}.}
\label{tab:rsgd_light_tasks}
\resizebox{\linewidth}{!}{%
\begin{tabular}{lccccccc}
\toprule
\textbf{Method} & \shortstack{\textbf{\# Unique $\mathcal{GS}$}\\\textbf{params}} & 
\shortstack{\textbf{\# Full $\mathcal{GS}$}\\\textbf{params}} & \textbf{CoLA} & \textbf{MRPC} & \textbf{RTE} & \textbf{STS-B} & \textbf{Avg.} \\
\midrule
\multicolumn{7}{c}{\textit{Standard (no sharing)}} \\
\midrule
$b_1{=}8,\ k_1{=}96,\ b_2{=}8,\ k_2{=}96$ & 1.38M & 1.38M & 61.67 & 89.22 & 75.45 & 90.52 & 79.22 \\
$b_1{=}16,\ k_1{=}48,\ b_2{=}32,\ k_2{=}24$ & 2.85M & 2.85M & 61.40 & 90.20 & 77.26 & 90.75 & 79.90 \\
$b_1{=}24,\ k_1{=}32,\ b_2{=}24,\ k_2{=}32$ & 2.85M & 2.85M & 60.57 & 89.22 & 76.90 & 90.68 & 79.34 \\
$b_1{=}32,\ k_1{=}24,\ b_2{=}32,\ k_2{=}24$ & 3.59M & 3.59M & \textbf{65.56} & 90.20 & \textbf{81.95} & \underline{91.04} & \textbf{82.19} \\
$b_1{=}48,\ k_1{=}16,\ b_2{=}24,\ k_2{=}32$ & 3.96M & 3.96M & 62.66 & \textbf{91.42} & 79.42 & 90.79 & \underline{81.07} \\
\midrule
\multicolumn{7}{c}{\textit{Share $B_1$}} \\
\midrule
$b_1{=}8,\ k_1{=}96,\ b_2{=}8,\ k_2{=}96$ & 1.01M & 1.38M & 61.07 & 89.22 & 75.45 & 90.10 & 78.96 \\
$b_1{=}16,\ k_1{=}48,\ b_2{=}32,\ k_2{=}24$ & 2.13M & 2.85M & 61.85 & \underline{90.93} & 76.90 & 90.59 & 80.07 \\
$b_1{=}24,\ k_1{=}32,\ b_2{=}24,\ k_2{=}32$ & 1.78M & 2.85M & 60.41 & 89.71 & 77.26 & 90.58 & 79.49 \\
$b_1{=}32,\ k_1{=}24,\ b_2{=}32,\ k_2{=}24$ & 2.17M & 3.59M & 62.57 & 90.44 & 78.34 & 90.64 & 80.50 \\
$b_1{=}48,\ k_1{=}16,\ b_2{=}24,\ k_2{=}32$ & 1.88M & 3.96M & 62.91 & 89.95 & \underline{80.51} & 90.71 & 81.02 \\
\midrule
\multicolumn{7}{c}{\textit{Share $B_2$}} \\
\midrule
$b_1{=}8,\ k_1{=}96,\ b_2{=}8,\ k_2{=}96$ & 1.01M & 1.38M & 63.07 & 89.46 & 77.26 & 90.18 & 79.99 \\
$b_1{=}16,\ k_1{=}48,\ b_2{=}32,\ k_2{=}24$ & 1.44M & 2.85M & \underline{63.77} & 89.46 & 77.26 & 90.57 & 80.27 \\
$b_1{=}24,\ k_1{=}32,\ b_2{=}24,\ k_2{=}32$ & 1.78M & 2.85M & 62.61 & 89.46 & 77.98 & 90.40 & 80.11 \\
$b_1{=}32,\ k_1{=}24,\ b_2{=}32,\ k_2{=}24$ & 2.17M & 3.59M & 62.83 & 89.95 & 78.34 & 90.81 & 80.48 \\
$b_1{=}48,\ k_1{=}16,\ b_2{=}24,\ k_2{=}32$ & 2.89M & 3.96M & 62.88 & 89.95 & 79.06 & \textbf{91.09} & 80.75 \\
\midrule
\multicolumn{7}{c}{\textit{Share both $B_1$ and $B_2$}} \\
\midrule
$b_1{=}8,\ k_1{=}96,\ b_2{=}8,\ k_2{=}96$ & 0.65M & 1.38M & 49.66 & 80.64 & 66.06 & 87.80 & 71.04 \\
$b_1{=}16,\ k_1{=}48,\ b_2{=}32,\ k_2{=}24$ & 0.72M & 2.85M & 60.34 & 87.50 & 76.17 & 90.42 & 78.61 \\
$b_1{=}24,\ k_1{=}32,\ b_2{=}24,\ k_2{=}32$ & 0.71M & 2.85M & 58.59 & 87.50 & 75.45 & 90.47 & 78.00 \\
$b_1{=}32,\ k_1{=}24,\ b_2{=}32,\ k_2{=}24$ & 0.76M & 3.59M & 61.10 & 88.73 & 78.70 & 90.78 & 79.83 \\
$b_1{=}48,\ k_1{=}16,\ b_2{=}24,\ k_2{=}32$ & 0.81M & 3.96M & 59.50 & 88.73 & 77.26 & 90.66 & 79.04 \\
\bottomrule
\end{tabular}%
}
\end{table}

\section{Conclusion}
We have studied the geometric properties of \( \mathcal{GS} \) matrices. For the case of orthogonal blocks, we present a robust framework for Riemannian optimization on \( \mathcal{GS} \)-manifolds. We enable efficient gradient computations via automatic differentiation. The complexity analysis confirms that this structured approach incurs limited overhead compared to standard training. Furthermore, we have obtained certain results for higher-order factorizations. Future work includes extending to higher‑order optimizers to outperform AdamW, handling different block constraints, and a more detailed investigation of factorizations with more than two block‑diagonal factors.

\appendix
\section{\texorpdfstring{\(\mathcal{G}\mathcal{S}\) decomposition with orthogonal blocks for \(m=3\)}{GS decomposition with orthogonal blocks for m=3}} \label{appendix:dimension_m3}

In this section of the appendix, we analyze the expressivity of the set \(\mathcal{A}_3^{\text{orth}}\), focusing on the special case involving Perfect Shuffle matrices.

First, we show that, for arbitrary permutation matrices \(P_2\) and \(P_1\) (and the block sizes in \(B_i\) are allowed to differ), $\mathcal{A}_3 \cap \mathrm{O}(N) = \mathcal{A}_3^{\mathrm{orth}}.$ To proceed, we need the following lemma, which allows us to absorb the permutation matrices into the \(B_i\).

\begin{lemma}\label{lem:diag-commute}
Let \(\{1,\dots,N\}=\bigsqcup_{\alpha=1}^{k}J_\alpha\) be a partition,
\(n_\alpha=|J_\alpha|\), and
\[
D=\operatorname{diag}(\lambda_1 I_{n_1},\dots,\lambda_k I_{n_k}),
\qquad \lambda_\alpha\neq\lambda_\beta \text{ for }\alpha\neq\beta.
\]
Then \(XD = DX \iff \)  \(X\) is block-diagonal with respect to this partition.
\end{lemma}

\begin{proof}
Write \(X=(X_{\alpha\beta})\) in blocks of sizes \(n_\alpha\times n_\beta\).
Since \(D\) is block-diagonal with scalar blocks,
\[
(XD)_{\alpha\beta}=\lambda_\beta X_{\alpha\beta},
\qquad
(DX)_{\alpha\beta}=\lambda_\alpha X_{\alpha\beta}.
\]
Thus \(XD = DX \iff (\lambda_\alpha-\lambda_\beta)X_{\alpha\beta}=0\) for all
\(\alpha,\beta\). For \(\alpha\neq\beta\) we have \(\lambda_\alpha\neq\lambda_\beta\),
hence \(X_{\alpha\beta}=0\). The converse is also immediate: $(\lambda_\alpha I) X_{\alpha\alpha} = \lambda_\alpha X_{\alpha\alpha} = X_{\alpha\alpha} (\lambda_\alpha I)$.
\end{proof}

Using the standard QR factorization of a matrix, one can extract a permutation matrix from the orthogonal factor. We need to obtain a modified factorization with respect to a block-diagonal partitioning.

\begin{corollary}\label{cor:modified-QR}
Let \(\mathcal{B}\) be the set of block-diagonal invertible matrices with respect to the partition above, and let \(P\) a permutation matrix. Then for every $B \in \mathcal{B}$ there exist \(U\in\mathcal B\cap\mathrm{O}(N)\) and an invertible upper-triangular \(R\) such that $BP=UPR$. Similarly, there exist \(U\in\mathcal B\cap\mathrm{O}(N)\) and an invertible upper-triangular \(R\) such that $PB=RPU.$
\end{corollary}

\begin{proof}
Choose \(D\) as in Lemma~\ref{lem:diag-commute}. Since \(B\in\mathcal B\), we have \(BD=DB\). Apply the ordinary QR factorization to \(BP\):
\(BP=VR\) with \(V\in\mathrm{O}(N)\) and \(R\) invertible and upper triangular. Put \(\widehat{D}:=P^\top DP\); the conjugation by a permutation matrix acts on a diagonal matrix by permuting its diagonal entries, and therefore \(\widehat{D}\) is diagonal as well. From \(BD=DB\),
\[
D(BP)=B(DP)=BP\,(P^\top DP)=(BP)\widehat{D},
\]
hence \(DVR=VR\widehat{D}\) and \(V^\top DV=R\widehat{D}R^{-1}\). The left-hand side is symmetric, the right-hand side is upper triangular, so both are diagonal. Since \(R,\widehat{D},R^{-1}\) are upper triangular,
\[
\operatorname{diag}(R\widehat{D}R^{-1}) = \operatorname{diag}(\widehat{D}),
\]
so \(R\widehat{D}R^{-1} = \widehat{D}\) and \(V^\top DV = \widehat{D}\).
Set \(U := V P^\top\). Then \(U \in \mathrm{O}(N)\) and
\[
U^\top D U = P V^\top D V P^\top = P\widehat{D}P^\top = D,
\]
so \(DU = UD\) and, by Lemma~\ref{lem:diag-commute}, \(U \in \mathcal B\).
Finally \(V = UP\), hence \(BP = UPR\).

Now we prove an analogous statement for the RQ decomposition. Apply the ordinary RQ factorization to \(PB\): \(PB = RV\) with \(V\in\mathrm{O}(N)\) and \(R\) invertible and upper triangular. Put \(\widehat{D} := PDP^\top\). From \(BD = DB\), we obtain $(PB)D=\widehat{D}(PB),$ hence \(RVD=\widehat{D}RV\) and \(VDV^\top = R^{-1}\widehat{D}R\). Again the left side is symmetric and the right side upper triangular, so
\(R^{-1}\widehat{D}R = \widehat{D}\) and \(VDV^\top = \widehat{D}\). Set \(U := P^\top V\). Then \(U \in \mathrm{O}(N)\) and $UDU^\top = P^\top VDV^\top P = D,$
so \(UD = DU\) and hence \(U \in \mathcal B\). Finally \(V = PU\), hence \(PB = RPU\).
\end{proof}

\begin{proposition} \label{prop:m3}
Let \(P_1,P_2\) be arbitrary fixed permutation matrices and let \(\mathcal{B}_i\) denote the sets of block-diagonal matrices with the prescribed square block sizes (the block sizes in \(B_i\) can be different). Then \(\mathcal A_3\cap\mathrm{O}(N) = \mathcal A_3^{\mathrm{orth}}\).
\end{proposition}

\begin{proof}
The inclusion $\mathcal{A}_3^{\mathrm{orth}} \subseteq \mathcal{A}_3 \cap \mathrm{O}(N)$ is trivial. We want to prove the reverse inclusion, i.e., for any permutation matrices \(P_1\) and \(P_2\), if \(Q = B_3 P_2 B_2 P_1 B_1 \in \mathcal{A}_3 \cap \mathrm{O}(N),\) then \(Q\) admits a decomposition of the form $Q=U_3 P_2 U_2 P_1 U_1$, where each \(U_i \in \mathcal{B}_i^{\mathrm{orth}}\) is block‑diagonal with orthogonal diagonal blocks. To construct such a decomposition, we apply modified QR and RQ factorizations.

Observe that $\mathcal{A}_3\cap \mathrm{O}(N) = \mathcal{A}_3^{\mathrm{inv}} \cap \mathrm{O}(N).$ Indeed, if a block were non-invertible, the determinant of the whole matrix -- being \(\pm 1\) times the product of the block determinants -- would be \(0\), contradicting orthogonality. Apply the QR part of Corollary~\ref{cor:modified-QR} to \(B_3 P_2\) and the
RQ part to \(P_1 B_1\). This gives
\[
B_3P_2=U_3P_2R_3,
\qquad
P_1B_1=R_1P_1U_1,
\]
with \(U_3\in\mathcal B_3\cap\mathrm O(N)\),
\(U_1\in\mathcal B_1\cap\mathrm O(N)\) and \(R_3,R_1\) invertible upper
triangular. Substituting,
\[
Q=U_3P_2\,\bigl(R_3B_2R_1\bigr)\,P_1U_1.
\]
Set \(U_2:=R_3B_2R_1\). Since
\[
U_2=P_2^\top U_3^\top Q\,U_1^\top P_1^\top
\]
is a product of orthogonal matrices, \(U_2\in\mathrm O(N)\).

Moreover, upper-triangular matrices are block upper triangular with respect to any partition into consecutive blocks, in particular with respect to the partition defining \(\mathcal B_2\). Since the class of block upper-triangular matrices is closed under multiplication, \(U_2 = R_3B_2R_1\) is block upper triangular with respect to \(\mathcal B_2\). Its inverse is therefore also block upper triangular, while \(U_2^{-1} = U_2^\top\) is block lower triangular; hence
\(U_2^{-1}\), and therefore \(U_2\), is block diagonal. Consequently
\(U_2\in\mathcal B_2\cap\mathrm O(N)\), and
\[
Q=U_3P_2U_2P_1U_1\in\mathcal A_3^{\mathrm{orth}}.
\]
\end{proof}

We now turn to the important particular case of the decomposition with \(m = 3\). Let \( p, q \ge 2 \) be integers and define \( N = pq \). Consider the \(\mathcal{GS}\) factorization with \(m = 3\) factors, where \(P_1 = P\) and \(P_2 = P^\top\) for a Perfect Shuffle matrix \(P\), and let all blocks be orthogonal. Specifically,
\begin{equation} \label{eq:m3_decomposition}
M = B_3 P^\top B_2 P B_1,    
\end{equation}
with \( B_1 \) and \( B_3 \) block-diagonal matrices comprising \( q \) orthogonal blocks of size \( p \times p \):
\[
B_1 = \begin{pmatrix}
V_1 & & \\
& \ddots & \\
& & V_q
\end{pmatrix}, \qquad
B_3 = \begin{pmatrix}
W_1 & & \\
& \ddots & \\
& & W_q
\end{pmatrix},
\]
with \( V_i, W_i \in \mathrm{O}(p) \).

The middle factor \( B_2 \) is a block-diagonal matrix consisting of \( p \) orthogonal blocks of size \( q \times q \):
\[
B_2 = \begin{pmatrix}
U_1 & & \\
& \ddots & \\
& & U_p
\end{pmatrix},
\qquad U_k \in \mathrm{O}(q).
\]
Denote the entries of the \( k \)-th block by \( U_k = (u_{ij}^{(k)})_{i,j=1}^q \). The matrix \( P \) is the Perfect Shuffle permutation.

To evaluate the explicit form of \( M \), we first compute the central conjugate term \( M' = P^\top B_2 P \). Applying Corollary \ref{corollary:similarity_PS}, the similarity transformation by the Perfect Shuffle matrix rearranges the entries of \( B_2 \) into a \( q \times q \) grid of blocks, where each block is a \( p \times p \) diagonal matrix:
\[
M' = \begin{pmatrix}
D_{11} & D_{12} & \dots & D_{1q} \\
D_{21} & D_{22} & \dots & D_{2q} \\
\vdots & \vdots & \ddots & \vdots \\
D_{q1} & D_{q2} & \dots & D_{qq}
\end{pmatrix}, \quad \text{where } D_{ij} = \operatorname{diag}\left(u_{ij}^{(1)}, u_{ij}^{(2)}, \dots, u_{ij}^{(p)}\right).
\]

Next, we multiply the three factors \( M = B_3 M' B_1 \):
\begin{equation} \label{eq:m3_form}
M = \begin{pmatrix}
W_1 D_{11} V_1 & W_1 D_{12} V_2 & \dots & W_1 D_{1q} V_q \\
W_2 D_{21} V_1 & W_2 D_{22} V_2 & \dots & W_2 D_{2q} V_q \\
\vdots & \vdots & \ddots & \vdots \\
W_q D_{q1} V_1 & W_q D_{q2} V_2 & \dots & W_q D_{qq} V_q
\end{pmatrix}.  
\end{equation}
This explicit block structure provides immediate insight into the set's expressivity. Consider the specific case where \( q = 2 \) (for any \( p \ge 2 \)). The matrices \( U_k \in \mathrm{SO}(2) \) can be parameterized by:
\[ U_k = \begin{pmatrix} \cos\theta_k & -\sin\theta_k \\ \sin\theta_k & \cos\theta_k \end{pmatrix}. \]
Consequently, the diagonal matrices become \( D_{11} = D_{22} = \operatorname{diag}(\cos\theta_1, \dots, \cos\theta_p) := C \) and \( D_{21} = -D_{12} = \operatorname{diag}(\sin\theta_1, \dots, \sin\theta_p) := S \). Substituting this into \( M \) yields:
\[
M = \begin{pmatrix} W_1 & 0 \\ 0 & W_2 \end{pmatrix}
\begin{pmatrix} C & -S \\ S & C \end{pmatrix}
\begin{pmatrix} V_1 & 0 \\ 0 & V_2 \end{pmatrix}.
\]
This is exactly the classical Cosine-Sine (CS) decomposition \cite[Theorem 2.5.3]{GL}, which guarantees that any orthogonal matrix \( Q \in \mathrm{O}(2p) \) can be factorized as above. Consequently, for \( q = 2 \), the \(\mathcal{GS}\) factorization yields a smooth manifold, and we have \( \mathcal{A}_3^{\text{orth}} = \mathcal{A}_3 \cap \mathrm{O}(N) = \mathrm{O}(N) \).

\begin{remark} \label{rem:dimensional_deficiency_m3}
However, for \( p > 1 \) and \( q > 2 \), \( \mathcal{A}_3^{\text{orth}} \) is a proper subset of \( \mathrm{O}(N) \).  Indeed, the dimension of its parameter space is \( q \cdot \frac{p(p-1)}{2} + q \cdot \frac{p(p-1)}{2} + p \cdot \frac{q(q-1)}{2} = qp(p-1) + \frac{pq(q-1)}{2}\), and the dimension of its image cannot exceed that of its domain. Given that \(\dim \mathrm{O}(N) = \frac{pq(pq - 1)}{2}\), the dimensional deficiency is
\[
\begin{aligned}
\Delta &= \dim \mathrm{O}(N) - \max \dim \mathcal{A}_3^{\text{orth}} = \frac{pq(pq - 1) - 2qp(p - 1) - pq(q - 1)}{2} \\
&= \frac{pq\left(pq - 1 - 2(p - 1) - (q - 1)\right)}{2} =\frac{pq}{2}(p-1)(q-2).
\end{aligned}
\]
Because \(\Delta > 0\) for \( q > 2 \), the \( m = 3 \) decomposition lacks the degrees of freedom required to span the full orthogonal group.

\end{remark}

We now wish to prove that the decomposition \eqref{eq:m3_decomposition} does not always admit a manifold structure. There is, however, an important clarification: we are speaking about the induced topology, because an artificial manifold topology can always be imposed on any continuum-sized set. Indeed, one can choose a bijection from \( \mathcal{A}_3 \cap \mathrm{O}(N) = \mathcal{A}_3^{\text{orth}} \) to $ \mathbb{R}^s$; this allows any topology on \(\mathbb{R}^s\) to be transferred to \(\mathcal{A}_3^{\text{orth}}\) (by declaring a subset open iff its image in \(\mathbb{R}^s\) is open). The same obstruction arises not only for the topological manifold structure, but also for the smooth manifold structure. Thus, the issue of nonexistence of a (smooth) manifold structure only arises when the topology is required to be the one induced from the ambient space -- namely, the subspace topology inherited from \(\mathrm{O}(N)\) with its standard Euclidean topology (equivalently, the topology inherited from \(\mathbb{R}^{N^2}\)).  In this induced topology (see \cite[page 601]{Lee}), a subset \(\mathsf{U} \subseteq \mathcal{A}_3^{\text{orth}}\) is open if and only if there exists an open subset \(\mathsf{V} \subseteq \mathbb{R}^{N^2}\) such that \(\mathsf{U} = \mathsf{V} \cap \mathcal{A}_3^{\text{orth}}\). For simplicity, we restrict to the case \(p = 2\) and \(q \ge3 \) with $q \bmod 4 \in \{2,3\}$, in which we will show that \(\mathcal{A}_3^{\text{orth}}\) fails to be a manifold under the topology induced from \(\mathbb{R}^{N^2}\). To do so, we need to show that the definition from \cite[Chapter 1, pages 2-3]{Lee} fails for \(\mathcal{A}_3^{\text{orth}}\). To disprove the manifold property, it suffices to find a point \(M_0 \in \mathcal{A}_3^{\mathrm{orth}}\) that admits no open neighborhood
homeomorphic to an open subset of any Euclidean space.
Choose \(Q_0 = (q_{ij}) \in \mathrm{SO}(q)\) with \(q_{ij} \neq 0\) for all \(i,j\) (such a matrix exists for every \( q \geq 3 \), for instance, one may take \( Q_0 = e^{t\Omega} \), where \( \Omega \in \mathfrak{so}(q) \) has no zero off-diagonal entries and \( t > 0 \) is sufficiently small), and set
\[
M_0=P^\top\begin{pmatrix}Q_0&0\\0&Q_0\end{pmatrix}P \in \mathcal{A}_3^{\text{orth}}.
\]

Thus, using corollary \ref{corollary:similarity_PS}, we obtain \((M_0)_{ij}=q_{ij}I_2\). Now by definition $X \in \mathcal{A}_3^{\text{orth}} \cap \mathbb{B}_{\varepsilon}(M_0) \iff X $ has the form \eqref{eq:m3_form} and $\|X-M_0\|_F < \varepsilon$. After squaring and applying a block decomposition, we obtain the equivalent condition
\begin{equation} \label{eq:m3_epsilon_ball}
\sum_{i,j=1}^{q} \left\| W_i \begin{pmatrix} \tilde{a}_{ij} & 0 \\ 0 & \tilde{\tilde{b}}_{ij} \end{pmatrix} V_j - q_{ij} I_2 \right\|_F^2 < \varepsilon^2,
\end{equation}
where \(W_i, V_j \in \mathrm{O}(2)\), and \(A=(\tilde{a}_{ij})\in\mathrm{O}(q)\), \(B=(\tilde{\tilde{b}}_{ij})\in\mathrm{O}(q)\) are the two \(q\times q\) diagonal blocks of the middle factor. The remainder of the proof is rather technical and consists of finding a parametrization that allows us to solve the inequality. We will absorb reflections into the middle factors. The following elementary lemma shows that the resulting row and column sign changes preserve orthogonality.

\begin{lemma}\label{lem:row-column-signs}
Let \(C=(c_{ij})\in\mathrm{O}(q)\), and let \(\rho_i,\kappa_j\in\{\pm1\}\). If
\[
\widetilde c_{ij}=\rho_i\kappa_j c_{ij},
\]
then \(\widetilde C=(\widetilde c_{ij})\in\mathrm{O}(q)\).
\end{lemma}

\begin{proof}
Let \(D_\rho = \operatorname{diag}(\rho_1,\ldots,\rho_q)\) and \(D_\kappa = \operatorname{diag}(\kappa_1,\ldots,\kappa_q)\). Then $\widetilde C = D_\rho C D_\kappa,$
and hence
\[
\widetilde C^\top\widetilde C
=D_\kappa^\top C^\top D_\rho^\top D_\rho C D_\kappa
=D_\kappa C^\top C D_\kappa
=I_q.
\]
\end{proof}

We now proceed to disproving the existence of a manifold structure. Let
\[
R(\theta)=\begin{pmatrix}\cos\theta&-\sin\theta\\ \sin\theta&\cos\theta\end{pmatrix},
\qquad
F=\begin{pmatrix}1&0\\0&-1\end{pmatrix}, \quad F^2 = I_2.
\]
Note that, due to the periodicity of trigonometric functions, the parametrization of $\mathrm{SO}(2)$ by angles is not unique; we will choose convenient representatives later. The main role of the matrix \(F\) is to absorb the negative determinant. Indeed, it is easy to show that left or right multiplication by \(F\) is a diffeomorphism. Also define \(\tau_i, \delta_j \in \{1,2\}\); then we can write \(W_i = R(\tilde{\alpha}_i)F^{\tau_i}\) and \(V_j = F^{\delta_j}R(\tilde{\beta}_j)\). We define $\tilde{b}_{ij}=(-1)^{\tau_i+\delta_j}\tilde{\tilde{b}}_{ij}$. Note that, by lemma \ref{lem:row-column-signs}, the matrix $\tilde{B} = (\tilde{b}_{ij})$ is also orthogonal, hence \eqref{eq:m3_epsilon_ball} can be rewritten in the following form:
\[
\sum_{i,j=1}^{q} \left\| \begin{pmatrix} \tilde{a}_{ij} & 0 \\ 0 & \tilde{b}_{ij} \end{pmatrix} - q_{ij} R(-\tilde{\alpha}_i)  R(-\tilde{\beta}_j)\right\|_F^2 < \varepsilon^2,
\]
Introduce \(\tilde{\theta}_{ij}=\tilde{\alpha}_i+\tilde{\beta}_j\) and observe that $\tilde{\theta}_{ij} - \tilde{\theta}_{i1} - \tilde{\theta}_{1j} + \tilde{\theta}_{11} = 0$. Since rotations commute, \(R(-\tilde{\alpha}_i) R(-\tilde{\beta}_j) = R(-\tilde{\theta}_{ij})\). Hence the inequality \eqref{eq:m3_epsilon_ball}  becomes
\begin{equation}\label{eq:m3_ball_angles}
\sum_{i,j=1}^{q}\left[
(\tilde{a}_{ij}-q_{ij} \cos \tilde{\theta}_{ij})^2 + (\tilde b_{ij}-q_{ij}\cos \tilde{\theta}_{ij})^2 + 2q_{ij}^2 \sin^2 \tilde{\theta}_{ij} \right] < \varepsilon^2.
\end{equation}

Note that from \eqref{eq:m3_ball_angles} we also deduce that $2q_{ij}^2 \sin^2 \tilde{\theta}_{ij} < \varepsilon^2$ and use it to choose convenient representatives of the angles. Since \(q_{ij} \neq 0 \), $\sin^2 \tilde{\theta}_{ij} \to 0$ for every \(i,j\) if $\varepsilon \to 0$. Hence, for sufficiently small \(\varepsilon\) there exists a unique \(n_{ij}\in\mathbb Z\) such that
\[
\tilde{\theta}_{ij}= \pi n_{ij}+ \theta_{ij},\qquad |\theta_{ij}|< \frac{\pi}{4}.
\]
Using \(\tilde{\theta}_{ij} - \tilde{\theta}_{i1} -\tilde{\theta}_{1j} + \tilde{\theta}_{11}=0\), we obtain
\[
\pi\left(n_{ij}-n_{i1}-n_{1j}+n_{11}\right)
+\theta_{ij}-\theta_{i1}-\theta_{1j}+\theta_{11}=0.
\]
The second term has absolute value strictly smaller than \(\pi\), whereas the first term is an integer multiple of \(\pi\). Hence
\[
n_{ij}=n_{i1}+\underbrace{n_{1j}-n_{11}}_{\Delta_j}.
\]
We now choose new representatives of the angles:
\[
\alpha_i= \tilde{\alpha}_i - \pi n_{i1},\qquad
\beta_j = \tilde{\beta}_j - \pi \Delta_j,\qquad a_{ij} = (-1)^{n_{ij}} \tilde{a}_{ij}, \qquad
b_{ij} =(-1)^{n_{ij}} \tilde b_{ij}.
\]
Then
\[
\alpha_i+\beta_j
=\tilde{\theta}_{ij}-\pi n_{ij}
=\theta_{ij} \implies |\alpha_i+\beta_j|<\pi/4.
\]
Moreover,
\[
R(\alpha_i) = (-1)^{n_{i1}}R(\tilde{\alpha}_i),\qquad
R(\beta_j) = (-1)^{\Delta_j}R(\tilde{\beta}_j).
\]
Since \(n_{ij}=n_{i1}+\Delta_j\),
we have $(-1)^{n_{ij}}=(-1)^{n_{i1}}(-1)^{\Delta_j}.$ 
Moreover,
\[
\begin{aligned}
R(\tilde{\alpha}_i)
\begin{pmatrix}
\tilde a_{ij}&0\\
0&\tilde b_{ij}
\end{pmatrix}
R(\tilde{\beta}_j) &= (-1)^{n_{i1}+n_{ij}+\Delta_j}
R(\alpha_i)
\begin{pmatrix}
a_{ij}&0\\
0&b_{ij}
\end{pmatrix}
R(\beta_j) \\
&=
R(\alpha_i)
\begin{pmatrix}
a_{ij}&0\\
0&b_{ij}
\end{pmatrix}
R(\beta_j),
\end{aligned}
\]
because $n_{i1}+n_{ij}+\Delta_j = 2n_{i1}+2\Delta_j$.
It remains to verify that the new middle factors remain orthogonal. Let
\[
D_n:=\operatorname{diag}\left((-1)^{n_{11}},\ldots,(-1)^{n_{q1}}\right),
\qquad
D_\Delta:=\operatorname{diag}\left((-1)^{\Delta_1},\ldots,(-1)^{\Delta_q}\right).
\]
Since \(n_{ij}=n_{i1}+\Delta_j\), the new matrices $A=(a_{ij}), B=(b_{ij})$ satisfy $A = D_n\widetilde A D_\Delta, B = D_n\widetilde B D_\Delta,$
where \(\widetilde A=(\tilde a_{ij})\) and \(\widetilde B=(\tilde b_{ij})\). Hence Lemma~\ref{lem:row-column-signs} implies that $A,B\in\mathrm{O}(q).$ Consequently, \eqref{eq:m3_ball_angles} can now be written as
\begin{equation}\label{eq:m3_ball_angles_normalized}
\begin{aligned}
\varepsilon^2 &>
\sum_{i,j=1}^{q} \left\| W_i \begin{pmatrix} \tilde{a}_{ij} & 0 \\ 0 & \tilde{\tilde{b}}_{ij} \end{pmatrix} V_j - q_{ij} I_2 \right\|_F^2 =
\sum_{i,j=1}^{q} \left\| R(\tilde{\alpha}_i)  \begin{pmatrix} \tilde{a}_{ij} & 0 \\ 0 & \tilde{b}_{ij} \end{pmatrix} R(\tilde{\beta}_j) - q_{ij} I_2 \right\|_F^2 \\
&= \sum_{i,j=1}^{q} \left\| \begin{pmatrix} a_{ij} & 0 \\ 0 & b_{ij} \end{pmatrix} - q_{ij} R(-\theta_{ij})\right\|_F^2
\\
&=
\sum_{i,j=1}^{q}\left[
(a_{ij}-q_{ij}\cos\theta_{ij})^2
+(b_{ij}-q_{ij}\cos\theta_{ij})^2
+2q_{ij}^2\sin^2\theta_{ij}
\right] 
\end{aligned}
\end{equation}
with \(|\theta_{ij}|<\pi/4\). Observe that from  \eqref{eq:m3_ball_angles_normalized} we deduce that $|a_{ij}-q_{ij}\cos\theta_{ij}|<\varepsilon, |b_{ij}-q_{ij}\cos\theta_{ij}|<\varepsilon,$
and $2q_{ij}^2 \sin^2 \theta_{ij} < \varepsilon^2$. The latter inequality implies $\theta_{ij}\to 0$ as $\varepsilon\to 0$. Hence $a_{ij}\to q_{ij}, 
b_{ij}\to q_{ij}.$ Indeed, since $\sin^2\theta_{ij} =  1 - \cos^2\theta_{ij} \geq 1- \cos\theta_{ij}$:
\[
|a_{ij}-q_{ij}|
\le
|a_{ij}-q_{ij}\cos\theta_{ij}| + |q_{ij}|\,|1-\cos\theta_{ij}| < \varepsilon + 
|q_{ij}| \frac{\varepsilon^2}{2q_{ij}^2}
\]
and the right-hand side tends to zero as $\varepsilon\to 0$. The argument for \(b_{ij}\) is identical. Since \( Q_0 \in \mathrm{SO}(q) \) and \( A, B \) lie in an \( \varepsilon \)-neighbourhood of \( Q_0 \), for sufficiently small \( \varepsilon \) we may assume that \( A \) and \( B \) also belong to \( \mathrm{SO}(q) \). Expanding \eqref{eq:m3_ball_angles_normalized}, using orthogonality of matrices ($\sum_{i,j} a_{ij}^2=q, \sum_{i,j}b_{ij}^2=q, \sum_{i,j}q_{ij}^2=q$), we obtain
\begin{equation*}
\begin{aligned}
\|M-M_0\|_F^2 
&= \sum_{i,j=1}^{q} \left[ a_{ij}^2 + b_{ij}^2 + 2q_{ij}^2 - 2q_{ij}(a_{ij}+b_{ij})\cos\theta_{ij} \right] \\
&= 4q - 2\sum_{i,j=1}^{q} q_{ij}(a_{ij}+b_{ij})\cos\theta_{ij}.
\end{aligned}
\end{equation*}
From this expression, the origin of the contradiction becomes clear: the inequality depends on the sum of the angles \(\theta_{ij}\) and the only remaining step is to rigorously establish the connection with matrices of bounded rank. Let $d:=\dim\mathfrak{so}(q)=\frac{q(q-1)}{2}.$

By the standard local property of the exponential map for Lie groups \cite[Proposition~20.8(f)]{Lee} or \cite[Proposition~2.3]{GOV1993}, and since left multiplication by \(Q_0\) is a diffeomorphism, it follows that for a sufficiently small neighborhood of \(Q_0\) there exists a diffeomorphism (namely, the composition of the exponential map with left translation by \(Q_0\))
\[
\psi:\mathcal U\subset\mathfrak{so}(q)\longrightarrow
\mathcal V\subset\mathrm{SO}(q),
\qquad \psi(H):=Q_0e^H,
\]
from a neighborhood of \(0\) onto a neighborhood of \(Q_0\). Consequently, if \(A\) and \(B\) are sufficiently close to \(Q_0\), there exist unique small elements \(\tilde{c}, \tilde{z} \in \mathfrak{so}(q)\) such that \(A = \psi(\tilde{c})\) and \(B = \psi(\tilde{z})\). Applying the symmetric change of variables, we define small elements \(c := (\tilde{c}+\tilde{z})/2\) and \(z := (\tilde{c}-\tilde{z})/2\) in \(\mathfrak{so}(q)\), which satisfy \(A = \psi(c+z)\) and \(B = \psi(c-z)\). In particular, $z = 0 \iff A = B,$ and interchanging \(A\) and \(B\) corresponds exactly to \(z \mapsto -z\). Indeed, since $\psi$ is a diffeomorphism, \(A=\psi(c+z)\) and \(B=\psi(c-z)\), we have \(A=B \iff c+z=c-z\), i.e. \(z=0\). Since \(A=\psi(c+z)\) and \(B=\psi(c-z)\), swapping \(A\) and \(B\) gives \(\psi(c-z)\) and \(\psi(c+z)\), which is exactly the same as replacing \(z\) by \(-z\).

We now rewrite the outer angular parameters. Set $\ell_i:=\alpha_i+\beta_1, i=1,\ldots,q,$
and $r_1:=0, r_j:=\beta_j-\beta_1, j=2,\ldots,q.$ Then $\alpha_i=\ell_i-\beta_1, \beta_j = r_j+\beta_1$, and $\theta_{ij}=\alpha_i+\beta_j = \ell_i+r_j.$

Thus \(2q-1\) combinations of angles are described by \((\ell_1,\ldots,\ell_q,r_2,\ldots,r_q)\), while \(\beta_1\) is the only remaining free angular parameter, which we want to isolate. Note that in the new variables
\[
M_{ij}
= R(\ell_i) \left[R(-\beta_1)
\begin{pmatrix}
a_{ij}&0\\
0&b_{ij}
\end{pmatrix}
R(\beta_1) \right]R(r_j).
\]
Using $\begin{pmatrix}
a_{ij}&0\\
0&b_{ij}
\end{pmatrix}
= \frac{a_{ij}+b_{ij}}{2}I_2 + \frac{a_{ij}-b_{ij}}{2} F$, we obtain
\[
R(-\beta_1)
\begin{pmatrix}
a_{ij}&0\\
0&b_{ij}
\end{pmatrix}
R(\beta_1)
=
\frac{a_{ij}+b_{ij}}{2}I_2 + \frac{a_{ij}-b_{ij}}{2}\underbrace{R(-\beta_1) F R(\beta_1)}_{T},
\]
Note that 
$
T =
\left(\begin{smallmatrix}
\cos \beta_1 & \sin \beta_1 \\
-\sin \beta_1 & \cos \beta_1
\end{smallmatrix}\right)
\left(\begin{smallmatrix}
\cos \beta_1 & -\sin \beta_1 \\
-\sin \beta_1 & -\cos \beta_1
\end{smallmatrix}\right)
=
\left(\begin{smallmatrix}
\cos^2\beta_1 - \sin^2\beta_1 & -2\cos\beta_1\sin\beta_1 \\
-2\sin\beta_1\cos\beta_1 & \sin^2\beta_1 - \cos^2\beta_1
\end{smallmatrix}\right).
$

Using $u:=
\begin{pmatrix}
\cos(2\beta_1)\\
-\sin(2\beta_1)
\end{pmatrix}
\in S^1$ we rewrite $T$ as $T(u) = 
\begin{pmatrix}
u_1&u_2\\
u_2&-u_1
\end{pmatrix},$
and therefore
\begin{equation}\label{eq:m3_compact_param}
M_{ij}
=
R(\ell_i)
\left[
\frac{a_{ij}+b_{ij}}{2}I_2
+
\frac{a_{ij}-b_{ij}}{2}T(u)
\right]
R(r_j).
\end{equation}
The representation \eqref{eq:m3_compact_param} separates the part independent of the remaining angle \(\beta_1\) from the part on which this angle acts. In particular, the first term depends only on \(A+B\), whereas the second term couples the difference \(A-B\) with a single point \(u\in S^1\). When \(A = B\), the dependence on \(u\) disappears completely.

Fix a linear diffeomorphism (and isomorphism) 
$\mathfrak{so}(q)\simeq\mathbb R^d, d = \frac{q(q-1)}2.$
Choose \(\eta>0\) sufficiently small so that
$\overline{\mathbb{B}}_{2\eta}^d(0)\subset\mathcal U$. Then every \(A,B\) sufficiently close to \(Q_0\) can be written uniquely as $A = \psi(c+z), B = \psi(c-z),$ with \(c,z\in\mathbb R^d\), \(\|c\|,\|z\|<\eta\). Consider the compact parameter space
\[
K=
[-\eta,\eta]^{2q-1}
\times
\overline{\mathbb B}_\eta(0)
\times
\overline{\mathbb B}_\eta(0)
\times
S^1,
\]
where the first factor contains $(\ell_1,\ldots,\ell_q,r_2,\ldots,r_q),$ the next two factors correspond to \(c,z\in\mathbb R^d\) and the last factor corresponds to \(u\in S^1\). The map \(f:K\to\mathcal{A}_3^{\mathrm{orth}}\) is the parametrization
defined by \eqref{eq:m3_compact_param}, i.e. $f(\ell,r,c,z,u):= M, \, M_{ij}=R(\ell_i)\Bigl[\tfrac{a_{ij}+b_{ij}}{2}I_2
+\tfrac{a_{ij}-b_{ij}}{2}T(u)\Bigr]R(r_j),$
where \((a_{ij})=\psi(c+z)\) and \((b_{ij})=\psi(c-z)\).
Introduce
\[
g:K\to\mathbb{R}^{2q-1+d}\times\mathbb{R}^{2\times d},
\qquad
g(\ell,r,c,z,u):=(\ell,r,c,uz^\top).
\]
We will use the following standard facts from topology \cite[Lemma~A.52(b), Theorem~A.31]{Lee}: a continuous surjection from a compact space onto a Hausdorff space is a quotient map; and if two quotient maps have the same fibres, then the induced map between their images is a homeomorphism. Since \(u \in S^1\), we have \(\operatorname{rank}(u z^\top) \le 1\), so the image of the last component lies in the determinantal cone  $\mathcal{C}_{2, d} = \{ Z \in \mathbb{R}^{2 \times d} : \operatorname{rank} Z \le 1 \}.$ Using the coincidence of fibres, we will show that the resulting space is locally homeomorphic to a set that is not homeomorphic to Euclidean space. We claim that \(f\) and \(g\) have the same fibres, i.e.
\[
f(p)=f(p') \quad \Longleftrightarrow\quad g(p) = g(p'),\qquad p,p'\in K.
\]

Indeed, \(\|M_{ij}-q_{ij}I_2\|_F^2<\varepsilon^2\) and continuity of the determinant give \(\det M_{ij}\to q_{ij}^2\); hence for small \(\varepsilon\) each \(M_{ij}\) is invertible. Also we have
\[
q_{ij}^{-1}M_{ij} = R(\ell_i+r_j) \, S_{ij}, \qquad
S_{ij}:= R(-r_j)
\left[ \frac{a_{ij}+b_{ij}}{2q_{ij}} I_2+
\frac{a_{ij}-b_{ij}}{2q_{ij}}T(u)
\right]
R(r_j).
\]
Since \(a_{ij},b_{ij}\to q_{ij}\neq0\), for small \(\varepsilon\) we have \(a_{ij}/q_{ij}>0\), \(b_{ij}/q_{ij}>0\). Hence, \(S_{ij}\) is symmetric positive definite. This is the polar decomposition, so \(R(\ell_i+r_j)\) is uniquely determined by \(M_{ij}\). Since \(|\ell_i+r_j|<\pi/4\), the angle \(\ell_i+r_j\) is unique, i.e. \(\ell_i+r_j = \ell'_i+r'_j\). Combined with \(r_1 = r'_1 = 0\), this yields \(\ell=\ell'\) and \(r=r'\) whenever \(f(p)=f(p')\). It remains to consider the pairs \((c, uz^\top)\) and \((c', u'z'^\top)\).

Cancelling the outer rotations gives, for all \(i,j\),
\[
\tfrac{a_{ij}+b_{ij}}{2}I_2+\tfrac{a_{ij}-b_{ij}}{2}T(u)
=
\tfrac{a'_{ij}+b'_{ij}}{2}I_2+\tfrac{a'_{ij}-b'_{ij}}{2}T(u').
\]
Taking traces (using \(\operatorname{tr}T(u)=0\)) gives \(a_{ij}+b_{ij}=a'_{ij}+b'_{ij}\), i.e. \(A+B=A'+B'\); subtracting the scalar parts yields
\begin{equation}\label{eq:diff}
(a_{ij}-b_{ij})\,T(u)=(a'_{ij}-b'_{ij})\,T(u').
\end{equation}

If $A=B$, then $(a'_{ij}-b'_{ij})T(u')=0$ and since $T(u') \neq 0$ for all $u$, we have $A' = B'$. If \(A=B\) and \(A'=B'\), then \(z=z'=0\), and injectivity
of \(\psi\) gives \(c=c'\). Therefore
\(uz^\top = u'z'^\top= 0\), so \(g(p) = g(p')\).

Assume \(A\neq B\) and choose \((i_0,j_0)\) with
\(\mathsf{\Delta}_{i_0j_0}\neq 0\), where
\(\mathsf{\Delta}_{ij}:=a_{ij}-b_{ij}\) and
\(\mathsf{\Delta}'_{ij}:=a'_{ij}-b'_{ij}\). Then \eqref{eq:diff} gives \(T(u')=\lambda T(u)\) with
\(\lambda=\mathsf{\Delta}_{i_0j_0}/\mathsf{\Delta}'_{i_0j_0}\), and taking determinants,
\[
-1=\det T(u')=\lambda^2\det T(u)=-\lambda^2,
\]
so \(\lambda=\pm 1\). Injectivity of \(u\mapsto T(u)\) and \(T(u') = \pm T(u)\) gives \(u'= \pm u\); substituting back into \eqref{eq:diff} yields
\(\mathsf{\Delta}_{ij} = \pm \mathsf{\Delta}'_{ij}\), i.e. \(\mathsf{\Delta}' = \pm \mathsf{\Delta}\). With \(A+B=A'+B'\), this forces either \((A',B')=(A,B)\) or \((A',B')=(B,A)\). If \((A',B')=(A,B)\), then 
\(\psi(c'+z')=\psi(c+z)\) and \(\psi(c'-z')=\psi(c-z)\),
and since $\psi$ is a diffeomorphism, we conclude that \(c'+z'=c+z\) and \(c'-z'=c-z\), so \(c'=c\), \(z'=z\).
If \((A',B')=(B,A)\), then \(\psi(c'+z')=\psi(c-z)\) and \(\psi(c'-z')=\psi(c+z)\), hence \(c'+z'=c-z\) and \(c'-z'=c+z\), so \(c'=c\), \(z'=-z\). In the second case \(u'z'^\top=(-u)(-z)^\top=uz^\top\), so \(g(p)=g(p')\) holds in both cases.

Conversely, suppose \(g(p)=g(p')\). Then, by definition of \(g\), the first
three components coincide, i.e. \(\ell=\ell'\), \(r=r'\), \(c=c'\), and $uz^\top=u'z'^\top.$ From \(c = c'\) we have \(A = \psi(c+z)\), \(A' = \psi(c+z')\),
\(B = \psi(c-z)\), \(B' = \psi(c-z')\). If \(z=0\), then \(u'z'^\top= uz^\top=0\). Since \(u'\in S^1\), this forces \(z'=0\). Hence
\(A'=\psi(c)=A\) and \(B'=\psi(c)=B\), so \(a'_{ij}=a_{ij}\) and \(b'_{ij}=b_{ij}\). In the block formula \eqref{eq:m3_compact_param} the coefficient \(a_{ij}-b_{ij}\) vanishes (since \(A=B\)), so the
term \(T(u)\) drops out and \(M_{ij}\) is independent of \(u\). Therefore \(M'_{ij}=M_{ij}\) regardless of the values of \(u,u'\), and \(f(p)=f(p')\).

If \(z\neq 0\), then choosing an index \(k\) with \(z_k\neq 0\) and comparing the \(k\)-th columns gives \(z_k u=z'_k u'\). Taking norms yields \(z'_k=\pm z_k\), hence \(u'=\pm u\) and
\(z'=\pm z\) with the same sign, i.e.\ either \((u',z')=(u,z)\) or
\((u',z')=(-u,-z)\). In the first case \(A'=A\), \(B'=B\),
\(T(u')=T(u)\), so \(M'=M\). In the second case \(z'=-z\), hence
\(A'=B\), \(B'=A\), and \(T(u')=-T(u)\); then
\[
\frac{a'_{ij}+b'_{ij}}{2}=\frac{b_{ij}+a_{ij}}{2}, \qquad
\frac{a'_{ij}-b'_{ij}}{2}T(u')
=\frac{b_{ij}-a_{ij}}{2}\left(-T(u)\right)
=\frac{a_{ij}-b_{ij}}{2}T(u),
\]
so again \(M'_{ij}=M_{ij}\). Thus \(f(p)=f(p')\). Hence the fibres of \(f\) and \(g\) coincide.

Since \(K\) is compact, \(f\) and \(g\) are continuous, and \(f(K)\) and \(g(K)\) are Hausdorff, both
\(f:K\to f(K)\) and \(g:K \to g(K)\) are quotient maps
\cite[Lemma~A.52(b)]{Lee}. Since their fibres coincide,
\cite[Theorem~A.31]{Lee} yields a unique homeomorphism
\(h:f(K) \to g(K)\) satisfying \(h(f(p))=g(p)\) for every \(p \in K\). To pass from these compact models to actual neighborhoods, it remains to check that \(M_0\) and \(0\) are interior points relative to the corresponding spaces. As shown earlier, when \(M \to M_0\), we have \(A, B \to Q_0\) and \(\theta_{ij} \to 0\). Hence $\ell_i = \theta_{i1} \to 0, r_j = \theta_{1j}-\theta_{11} \to 0$. Since \(\psi^{-1}\) is continuous and \(\psi^{-1}(Q_0)= 0\), the convergences \(A\to Q_0\) and \(B\to Q_0\) imply \(c+z=\psi^{-1}(A)\to 0\) and \(c-z=\psi^{-1}(B)\to 0\), hence \(c,z\to 0\). Consequently, for sufficiently small \(\varepsilon>0\), every
\(M\in\mathcal A_3^{\mathrm{orth}}\cap\mathbb B_\varepsilon(M_0)\)
admits a normalized representation \(M=f(\ell,r,c,z,u)\) with $|\ell_i|,|r_j|<\eta, \|c\|,\|z\|<\eta.$
No additional restriction is needed for \(u\), since \(u\in S^1\) by construction and the entire sphere \(S^1\) is included in the last factor of \(K\). Hence $\mathcal{A}_3^{\mathrm{orth}} \cap \mathbb{B}_\varepsilon(M_0) \subset f(K)$. The left-hand side is an open neighborhood of \(M_0\) in the subspace topology of \(\mathcal{A}_3^{\mathrm{orth}}\) (by definition); thus \(M_0\) is an interior point of \(f(K)\) relative to \(\mathcal{A}_3^{\mathrm{orth}}\). To obtain the corresponding statement for \(g(K)\), observe that
\[
\left\{uz^\top:u\in S^1,\ \|z\|\le\eta\right\}
=
\mathcal C_{2,d}\cap\overline{\mathbb B}_\eta^{2d}(0).
\]
Indeed, \(\|uz^\top\|_F = \sqrt{\sum_{i=1}^2\sum_{j=1}^d (u_i z_j)^2} = \sqrt{\left(\sum_{i=1}^2 u_i^2\right)\left(\sum_{j=1}^d z_j^2\right)}
=\|u\|\,\|z\|= \|z\|\), while every nonzero matrix \(Z\in\mathcal C_{2,d}\) can be written as \(Z=uz^\top\) with \(u\in S^1\). Therefore $g(K)=[-\eta,\eta]^{2q-1}\times \overline{\mathbb{B}}_\eta^d(0)\times
\left(\mathcal{C}_{2,d}\cap\overline{\mathbb{B}}_\eta^{2d}(0)\right).$ In particular, $\mathsf{W}:=
(-\eta,\eta)^{2q-1}\times\mathbb{B}_\eta^d(0)\times
\left(\mathcal{C}_{2,d}\cap\mathbb{B}_\eta^{2d}(0)\right)$ is an open neighborhood of \(0\) in
\(Y:= \mathbb{R}^e\times \mathcal{C}_{2,d}\), and \(\mathsf{W} \subset g(K)\).

Finally, choose \(p_0=(0,0,0,0,u_0)\in K\) with arbitrary \(u_0\in S^1\). Then \(f(p_0) = M_0\) and \(g(p_0) = 0\), so \(h(M_0) = 0\). Let $\mathsf{U}_0:= \mathcal{A}_3^{\mathrm{orth}} \cap \mathbb{B}_\varepsilon(M_0)$. Since \(\mathsf{U}_0\subset f(K)\), the set \(h(\mathsf{U}_0)\) is open in \(g(K)\) and contains \(0\). Hence \(\mathsf{W}_0 := h(\mathsf{U}_0) \cap \mathsf{W}\) is an open neighborhood of \(0\) in \(Y\): \(h(\mathsf{U}_0)\) being open in \(g(K)\) means \(h(\mathsf{U}_0) = \mathsf{O} \cap g(K)\) for some open \(\mathsf{O} \subset Y\), so \(\mathsf{W}_0 = \mathsf{O} \cap \mathsf{W}\) is open in \(Y\), and it contains \(0\) because \(0 \in h(\mathsf{U}_0) \cap \mathsf{W}\). Its preimage \(\mathsf U:=h^{-1}(\mathsf W_0)\) is open in \(f(K)\) and contained in the open set \(\mathsf U_0\); hence it is open in \(\mathcal A_3^{\mathrm{orth}}\), and the restriction $h|_\mathsf{U}:\mathsf{U}\longrightarrow \mathsf{W}_0$ is a homeomorphism. Consequently \((\mathcal{A}_3^{\mathrm{orth}},M_0)\) is locally homeomorphic to \((\mathbb{R}^e\times\mathcal{C}_{2, d},0)\), with \(e=2q-1+d\), \(d = \tfrac{q(q-1)}{2}\), and \(\mathcal{C}_{2, d} = \{Z\in\mathbb{R}^{2\times d}:\operatorname{rank} Z \le 1\}\). Since being a topological manifold at a point is a local property and
\[
(\mathcal{A}_3^{\mathrm{orth}},M_0)
\cong_{\mathrm{loc}}
(\mathbb{R}^e\times\mathcal{C}_{2,d},(0,0)),
\]
it is enough to study sufficiently small open neighborhoods of \((0,0)\) in
\(\mathbb{R}^e\times\mathcal{C}_{2,d}\). For simplicity, we consider the case when \(d\) is odd (equivalently, $q \bmod 4 \in \{2,3\}$), and we conclude the absence of an induced manifold structure by the following lemma. 

\begin{lemma}\label{lem:C2d_not_manifold_odd}
Let \(d\ge3\) be odd, \(e\ge0\), and $\mathcal C_{2,d}:=\{Z\in\mathbb R^{2\times d}:\operatorname{rank}Z\le1\}$. Then \(\mathbb R^e\times\mathcal C_{2,d}\) is not a topological manifold at \((0,0)\).
\end{lemma}

\begin{proof}
Suppose, to the contrary, that \(\mathbb R^e\times\mathcal C_{2,d}\) is a topological manifold at \((0,0)\). Then there exists an open coordinate neighborhood \(\mathbf U\ni(0,0)\) homeomorphic to an open subset of some Euclidean space. For sufficiently small \(\delta,\rho>0\),
\[
\mathbb{B}_\delta^e(0)\times V_\rho\subset\mathbf U,\qquad V_\rho:=\mathcal C_{2,d} \cap \mathbb{B}_\rho^{2d}(0).
\]
Note that 
\(
V_\rho\setminus\{0\}
=
\mathcal C_{2,d}\cap \left(\mathbb{B}_\rho^{2d}(0)\setminus\{0\}\right)
\) is also open (as an open set with the closed point \(\{0\}\) removed). Hence $\mathbb{B}_\delta^e(0)\times \left(V_\rho\setminus\{0\}\right)\subset\mathbf U$
is itself homeomorphic to an open subset of a Euclidean space and is therefore orientable.

Now put \(L_d:=\mathcal C_{2,d}\cap S^{2d-1}\). Since \(\mathcal C_{2,d}\) is a cone, radial coordinates give a homeomorphism
\[
V_\rho\setminus\{0\}\cong(0,\rho)\times L_d,\qquad
Z\longmapsto\left(\|Z\|_F,\frac{Z}{\|Z\|_F}\right),
\]
with inverse \((r,Y)\mapsto rY\). Moreover, define the continuous map \(\Phi: S^1 \times S^{d-1} \to L_d\), \((u,v) \mapsto u v^\top\). We now prove that it is surjective. Let \(Z \in L_d\), then \(\operatorname{rank} Z = 1\) and \(\|Z\|_F = 1\). Since \(Z\) has rank one, there exist nonzero vectors \(a \in \mathbb R^2\) and \(b \in \mathbb R^d\) such that \(Z = a b^\top\). Consider $u=\frac a{\|a\|}\in S^1, v=\frac b{\|b\|}\in S^{d-1}$. Because $\|Z\|_F=\|ab^\top\|_F=\|a\|\,\|b\|=1$, we obtain $uv^\top=\frac{ab^\top}{\|a\|\,\|b\|} = Z$.
Thus \(\Phi\) is surjective. Let
\[
\pi:S^1\times S^{d-1}\to \left(S^1\times S^{d-1}\right)/\sim,
\qquad
\pi(u,v)=[u,v],
\]
where \((u,v)\sim(-u,-v)\). Since $\Phi(-u,-v)=(-u)(-v)^\top=uv^\top=\Phi(u,v),$ the map \(\Phi\) is constant on the fibres of \(\pi\). Hence, by the universal property of the quotient topology \cite[Theorem~A.30]{Lee}, there exists a unique continuous map
\[
\overline{\Phi}:\left(S^1\times S^{d-1}\right)/\sim\ \longrightarrow L_d
\]
such that \(\Phi=\overline{\Phi}\circ\pi\); explicitly,
$\overline{\Phi}([u,v])=uv^\top$. The map \(\overline{\Phi}\) is surjective because \(\Phi\) is surjective. To prove injectivity, suppose that $uv^\top=u'v'^\top$. Multiplying on the right by \(v\) gives \(u(v^\top v)=u'(v'^\top v)\). Since \(v\in S^{d-1}\), \(v^\top v=\|v\|^2=1\), so \(u=(v'^\top v)u'\). Taking norms and using \(\|u\|=\|u'\|=1\) yields \(1=|v'^\top v|\). By Cauchy-Schwarz, \(|v'^\top v|\le\|v'\|\,\|v\|=1\), and equality holds iff \(v'\) and \(v\) are linearly dependent. Since both are unit vectors, this means \(v'=\pm v\): indeed, linear dependence gives \(v'=\lambda v\), and taking norms yields \(1 = \|v'\| = |\lambda|\,\|v\| = |\lambda|\), so \(\lambda = \pm1\). Hence \(v'^\top v=\pm1\), and from \(u=(v'^\top v)u'\) we get \(u=\pm u'\) with the same sign. Therefore \((u',v')=(u,v)\) or \((u',v')=(-u,-v)\), proving that $\overline{\Phi}$ on the quotient is injective.

Since \(S^1\times S^{d-1}\) is compact, its quotient \(\left(S^1\times S^{d-1}\right)/\sim\) is compact, while \(L_d\subset\mathbb R^{2\times d}\) is Hausdorff. Therefore \(\overline{\Phi}\) is a homeomorphism by \cite[Lemma~A.52(d)]{Lee}. Consequently, \(L_d\) is homeomorphic to the quotient $\left(S^1\times S^{d-1}\right)/\left((u,v)\sim(-u,-v)\right)$. This quotient, along with much more general objects, is studied in \cite{Davis2009}; in his notation, it is \(P_{(1,d-1)}\). By \cite[Corollary~3.8]{Davis2009}, \(P_{\bar{n}}\) is orientable if and only if \(|\bar{n}|+r\) is even. Here \(\bar{n}=(1,d-1)\), so \(|\bar{n}|=d\) and \(r=2\); since \(d\) is odd, \(L_d\) is nonorientable. Therefore
\[
B_\delta^e(0)\times\left(V_\rho\setminus\{0\}\right)
\cong B_\delta^e(0)\times(0,\rho)\times L_d
\]
is nonorientable \cite[Section 3.3, Exercise 5]{Hatcher2002}. This contradicts the fact that it is homeomorphic to an open subset of Euclidean space, since it possesses an orientation \cite[Proposition~15.11]{Lee} and orientation is a homeomorphism invariant \cite[Theorem~3.26 \& Section~2.1, page~108]{Hatcher2002}. Hence \(\mathbb R^e\times\mathcal C_{2,d}\) is not a topological manifold at \((0,0)\).
\end{proof}

To summarize the argument, we combine two local descriptions of the same neighbourhood of \(M_0\). We have already proved that there exist neighborhoods \(\mathsf{U}_0\) of \(M_0\) and \(\mathsf{W}_0\) of \((0,0)\) such that \(\mathsf{U}_0\cong \mathsf{W}_0\subset \mathbb{R}^e\times \mathcal{C}_{2,d}\). Suppose, for contradiction, that \(\mathcal{A}_3^{\mathrm{orth}}\) is a manifold at \(M_0\). Then there is another neighbourhood \(\mathsf{U}_1\) homeomorphic to an open subset of \(\mathbb{R}^s\). Replacing \(\mathsf{U}_0\) by \(\mathsf{U}_0\cap \mathsf{U}_1\), both descriptions apply simultaneously. The image of this intersection in \(\mathbb{R}^e\times\mathcal{C}_{2,d}\) contains \(\mathbb{B}_\delta^e(0) \times V_\rho\) for small \(\delta,\rho>0\). After deleting the zero matrix, we obtain $\mathbb{B}_\delta^e(0)\times(V_\rho\setminus\{0\})
\cong
\mathbb{B}_\delta^e(0)\times(0,\rho)\times L_d.$ For odd \(d\), \(L_d\cong P_{(1,d-1)}\) is nonorientable, so this open set is nonorientable. Its preimage in \(\mathcal{A}_3^{\mathrm{orth}}\) is an open subset of \(\mathsf{U}_1\), hence homeomorphic to an open subset of \(\mathbb{R}^s\), which is orientable. Contradiction. Thus \(\mathcal{A}_3^{\mathrm{orth}}\) is not a topological manifold at \(M_0\).

\section{\texorpdfstring{Geodesics in \(\mathcal{GS}\)-orthogonal matrices for \(k_1 \ge b_2\) and \(m=2\)}{Geodesics in GS-orthogonal matrices for k1 >= b2 and m=2}}
\label{appendix:geodesics}
The fact that \(\Phi\) is a covering map for \(\mathcal{A}_m^{\mathrm{SO}}\) under the condition \(k_1 \ge b_2\) is established in \cite[Theorem~3 and Appendix~B]{OF}. We now prove that $\Phi$ is a local isometry with respect to the standard product metric on $G = \mathcal{B}_2 \times \mathcal{B}_1$ (where \(\mathcal{B}_i = [\mathrm{SO}(b_i)]^{k_i}\)) and the ambient Frobenius metric on $\mathcal{A}_2^{\mathrm{orth}}$. Using this fact, we prove that the geodesics on the manifold of \(\mathcal{GS}\)-orthogonal matrices are projections of geodesics of the group \(G\), where the geodesics decompose into blockwise geodesics and are explicitly known. For this, we need the following proposition and definitions.

\begin{definition}[pages 330-331 in \cite{Lee},  Proposition, Definition 1.106 in \cite{GHL}]
Let $F: M \to N$ be a smooth map and $g$ a Riemannian metric on $N$. The \emph{pullback metric} $F^*g$ on $M$ is defined by \((F^*g)_p(v,w) = g_{F(p)}(\mathrm{d}F_p(v), \mathrm{d}F_p(w)), p \in M,\; v,w \in T_pM.\)
\end{definition}

\begin{definition}[Definition 2.5 in \cite{GHL}, page 24 in \cite{Leeriem} or Definition 1.4.5 in \cite{Jost17}] \label{def:local_isometry}
A smooth map $F: (M,g_M) \to (N,g_N)$ between Riemannian manifolds is a local isometry if for every $p \in M$ the differential $\mathrm{d}F_p: T_pM \to T_{F(p)}N$ is a linear isometry, i.e., $F^*g_N = g_M$.
\end{definition}

\begin{definition}[Definition 1.83 in \cite{GHL} or page 106 in \cite{Lee}]
Let \( M \) and \( \tilde{M} \) be two manifolds. A map \( p : \tilde{M} \to M \) is a smooth covering map if:
\begin{itemize}
\item \( p \) is smooth and surjective,

\item for every point \( x \in M \), there exists a neighborhood \( U \) of \( x \) in \( M \) such that \( p^{-1}(U) \) is a disjoint union \( \bigcup_{i \in I} U_i \) of open subsets of \( \tilde{M} \), and for each \( i \in I \), the restriction \( p: U_i \to U \) is a diffeomorphism.
\end{itemize}
\end{definition}

\begin{definition}[Definition 2.17 in \cite{GHL} or page 27 in \cite{Leeriem}] \label{def:riemannian_covering_map}
Let \( (M, g) \) and \( (\tilde{M}, h) \) be two Riemannian manifolds. A map \( p: \tilde{M} \to M \) is a Riemannian covering map if it is a smooth covering map and a local isometry.
\end{definition}

\begin{proposition}[Proposition 2.81 in \cite{GHL}  or Lemma 11.6 in \cite{Leeriem}] \label{prop:geodesics_projection}
Let \( p: (N, h) \to (M, g) \) be a Riemannian covering map. The geodesics of \( (M, g) \) are the projections of the geodesics of \( (N, h) \), and the geodesics of \( (N, h) \) are the liftings of those of \( (M, g) \).
\end{proposition}

For \(\mathcal{B}_i = [\mathrm{SO}(b_i)]^{k_i}\), we endow each \(\mathrm{SO}(b_i)\) with the standard bi‑invariant metric induced by the Frobenius inner product \(\langle X,Y\rangle = \operatorname{tr}(X^\top Y)\). The product Riemannian metric on the Lie group $G = \mathcal{B}_2 \times \mathcal{B}_1$ is given by the direct sum of these metrics. For any two tangent vectors $X = (B_2 C_2, B_1 C_1)$ and $Y = (B_2 D_2, B_1 D_1)$ at a point $(B_2, B_1) \in G$, the metric is explicitly given by:
\begin{equation} \label{eq:product_metric}
g_{(B_2, B_1)}(X, Y) = \sum_{i=1}^{k_2} \operatorname{tr}\left( (C_2^{(i)})^\top D_2^{(i)} \right) + \sum_{j=1}^{k_1} \operatorname{tr}\left( (C_1^{(j)})^\top D_1^{(j)} \right).
\end{equation}
For this metric, we show that the group action \(\Phi\left((B_2, B_1), x\right) = B_2 x B_1^{\top}\) is a Riemannian covering map. By Proposition~\ref{prop:geodesics_projection}, this allows us to obtain the geodesics of \(\mathcal{GS}\)-orthogonal matrices as projections of block-wise geodesics in \(G\). 
As shown in \cite[Appendix~B]{OF}, under the condition \(k_1 \ge b_2\), the map \(\Phi\) is a smooth covering map, which guarantees the existence of a unique Riemannian metric on \(G\) making \(\Phi\) a Riemannian covering map. Below, we rigorously prove that the standard product Frobenius metric \eqref{eq:product_metric} is precisely this unique metric, thereby establishing that \(\Phi\) is a local isometry.

\begin{proposition}
\label{prop:phi_isometry}
Assume $k_1 \ge b_2$. Then the map $\Phi: G \to \mathcal{A}_2^{\mathrm{orth}}$, $\Phi(B_2,B_1) = B_2 P_1 B_1^\top$, is a local isometry, where $G = \mathcal{B}_2 \times \mathcal{B}_1$ is endowed with the product of the standard Frobenius metrics, and $\mathcal{A}_2^{\mathrm{orth}}$ inherits the ambient Frobenius metric.
\end{proposition}

\begin{proof}
It suffices to verify the equality of norms at each point, since the equality of the metrics follows by polarization:
$$
\begin{aligned}
g_N(\mathrm{d}F_p(v), \mathrm{d}F_p(w)) &= \frac{1}{4}\left(\|\mathrm{d}F_p(v) + \mathrm{d}F_p(w)\|_{g_N}^2 - \|\mathrm{d}F_p(v) - \mathrm{d}F_p(w)\|_{g_N}^2\right) \\
&= \frac{1}{4}\left(\|\mathrm{d}F_p(v + w)\|_{g_N}^2 - \|\mathrm{d}F_p(v - w)\|_{g_N}^2\right) \\
&= \frac{1}{4}\left(\|v + w\|_{g_M}^2 - \|v - w\|_{g_M}^2\right) = g_M(v, w).
\end{aligned}
$$
Let $(B_2,B_1) \in G$ and take a tangent vector \(X = (B_2 C_2,\; B_1 C_1) \in T_{(B_2,B_1)} G,\) where $C_2 \in \mathfrak{so}_{\text{block}}(b_2, k_2)$ and $C_1 \in \mathfrak{so}_{\text{block}}(b_1, k_1)$ are block‑diagonal skew‑symmetric matrices. 
Define $\Omega_2 := B_2 C_2 B_2^\top$ and $\Omega_1 := B_1 C_1 B_1^\top$. The differential of $\Phi$ acts as \(\mathrm{d}\Phi_{(B_2,B_1)}(B_2 C_2,\; B_1 C_1) = (B_2 C_2) P_1 B_1^{\top} + B_2 P_1 (B_1 C_1)^{\top}.\) Since $(B_1 C_1)^{\top} = -C_1 B_1^{\top}$, and $C_1 B_1^{\top} = B_1^{\top} \Omega_1$ (because $C_1 = B_1^{\top} \Omega_1 B_1$), we obtain \(\mathrm{d}\Phi(X) = B_2 C_2 P_1 B_1^{\top} - B_2 P_1 C_1 B_1^{\top}.\)
Now compute the squared Frobenius norm. Left and right multiplication by orthogonal matrices are isometries, hence
\[
\|\mathrm{d}\Phi(X)\|_F^2 = \| B_2^{\top} \mathrm{d}\Phi(X) B_1 P_1^\top\|_F^2 = \| C_2 - P_1 C_1 P_1^\top \|_F^2 = \| C_2\|_F^2 + \|C_1\|_F^2.
\]
The last equality follows by reasoning analogous to that in Section \ref{subsec:proj_tangent}. Also,
\[
\|X\|_G^2 = \| B_2 C_2 \|_F^2 + \| B_1 C_1 \|_F^2 =  \sum_{i=1}^{k_2} \|C_2^{(i)}\|_F^2 + \sum_{j=1}^{k_1} \|C_1^{(j)}\|_F^2 = \|C_2\|_F^2 + \|C_1\|_F^2.
\]
Thus $\|\mathrm{d}\Phi(X)\|_F^2 = \|X\|_G^2 $ for every tangent vector $X$. Polarizing this identity yields $\Phi^* g = h$, i.e., $\Phi$ is a local isometry.
\end{proof}

\begin{corollary}
Being a local isometry --- and, even more strongly, a Riemannian covering map --- is a powerful condition. From it follow the preservation of curve lengths (see \cite[Exercise 13.24]{Lee}), the invariance of the curvature tensor (see \cite[Lemma 7.2]{Leeriem}), and other analogous statements.
\end{corollary}

We also give a simpler proof that \(\Phi\) is a Riemannian covering map, using similar statements but a more direct argument that avoids dependence on \cite{OF}.

\begin{definition}[Definition 2.101 in \cite{GHL}, page 108 in \cite{Leeriem} or Definition 1.7.1 in \cite{Jost17}] \label{def:geodesically_complete}
A Riemannian manifold \((M,g)\) is said to be geodesically complete if every geodesic can be extended to a geodesic defined on all of \(\mathbb{R}\).
\end{definition}

\begin{theorem}[Hopf-Rinow (see Theorem 6.13 in \cite{Leeriem}, Theorem 2.103 \& Corollary 2.105 in \cite{GHL} or Theorem 1.7.1 in \cite{Jost17})] \label{thm:Hopf-Rinow}
A connected Riemannian manifold is geodesically complete if and only if it is complete as a metric space.  
\end{theorem}

\begin{proposition}[Lemma 11.6 in \cite{Leeriem} or Proposition 2.106 in \cite{GHL}] \label{prop:Riemannian_covering_map}
Let \( p : (M, g) \to (N, h) \) be a local isometry. If \((M, g)\) is geodesically complete, then \( p \) is a Riemannian covering map and \((N, h)\) is also geodesically complete.
\end{proposition}

Since the group \(G\) is complete as a metric space (as a Cartesian product of complete (compact) spaces) and is connected, Theorem~\ref{thm:Hopf-Rinow} implies that \(G\) is geodesically complete (see Definition~\ref{def:geodesically_complete}). Moreover, the group action \(\Phi\) is a local isometry by Proposition~\ref{prop:phi_isometry}. Consequently, by Proposition~\ref{prop:Riemannian_covering_map}, the projections of geodesics of \(G\) are geodesics in the orbit, i.e., in the manifold of \(\mathcal{GS}\)-orthogonal matrices. It remains to show that the artificial transposition (introduced to ensure a well‑defined Lie group action) has no actual effect. Given matrices \(L_1 P R_1 = L_1 P (R_1^\top)^\top\) and \(L_2 P R_2 = L_2 P (R_2^\top)^\top\), an extra transposition is required after connecting \(R_1^\top\) and \(R_2^\top\). As shown in \cite[Proposition 10]{OF}, this is equivalent to simply connecting the blocks.

\begin{remark}
The inclusion of fixed orthogonal boundary matrices (i.e., permutations $P_{\text{in}}$ and $P_{\text{out}}$) yields a global isometry $\widetilde{\Phi}(B_2,B_1) = P_{\text{out}} B_2 P_1 B_1^\top P_{\text{in}}$. Consequently, by the same reasoning, when \(k_1 \geq b_2\) the geodesics in this space are also obtained from blockwise geodesics.
\end{remark}

\section{Hierarchical and Higher-order \texorpdfstring{\(\mathcal{GS}\)}{GS}-Decompositions}
\label{appendix:hierarchical_gs}

Suppose that each block \(V\) inside a block‑diagonal factor \(B_i\) is itself parameterized as a smaller \(\mathcal{GS}\) matrix (with the same \(m\) for every block). For example, for \(m=2\) we might write \(V = \widetilde{B}_2 \widetilde{P} \widetilde{B}_1\). Substituting such expansions back into every block of the global matrix algebraically flattens the nested structure into a standard, higher‑order \(\mathcal{GS}\) decomposition of increased depth \(m\). The inner permutations \(\widetilde{P}\) systematically assemble into global block‑permutation matrices, which are themselves permutation matrices. However, note that if each \(\widetilde{P}\) is a Perfect Shuffle matrix, the resulting global permutation is a block‑perfect shuffle, which globally is not a standard perfect shuffle matrix, while the inner block‑diagonal matrices concatenate into global block‑diagonal factors.

This flattening principle is particularly interesting when applied to the $m=3$ CS-decomposition from Appendix \ref{appendix:dimension_m3} ($q=2$). If we recursively decompose the outer \(N/2 \times N/2\) blocks (provided \(N\) is a multiple of 4) using the same \(m=3\) pattern, we have several options: we may apply this decomposition only to \(B_3\), only to \(B_1\), or to both. Observe that at each step we obtain several new decompositions of the orthogonal group; however, as demonstrated in Remark~\ref{rem:dimensional_deficiency_m3}, no excess degrees of freedom are added. Specifically, for a matrix of size \(N = 2^k p\), the choice of which block-diagonal factors to recursively decompose determines different decompositions. When \(N = 2^k\), we can descend all the way down to base blocks of size \(2 \times 2\), yielding an alternation of block‑diagonal matrices (with \(2 \times 2\) blocks) and permutation matrices that possess a block‑perfect shuffle structure. These decompositions are memory‑heavy and impractical for learning orthogonal matrices, but they resemble Givens matrices and may find applications in other domains.

Now we demonstrate \(\mathcal A_m^{\mathrm{orth}} \neq \mathcal A_m \cap \mathrm{O}(N)\) for factorizations with $m = 4$ factors. Specifically, for $m=4$, we construct a matrix $Q \in \mathcal{A}_4 \cap \mathrm{O}(N)$ that does not belong to $\mathcal{A}_4^{\mathrm{orth}}$. We consider the minimal case: $N=6$, with $m=4$ factors. Each factor consists of three blocks of size $2 \times 2$. We fix the permutation matrix between every layer to be the standard Perfect Shuffle $P = P_{(2,3)}$. 

We will construct a counterexample in a partially explicit way. By definition, $Q \in \mathcal{A}_4^{\mathrm{orth}}$ admits a factorization $Q = U_4 P U_3 P U_2 P U_1$, where $U_i \in [\mathrm{O}(2)]^3$. Observe that by isolating the two inner layers, we have:
\begin{equation*}
\begin{pmatrix}
(U_4^\top)_{11}\,Q_{11}\,(U_1^\top)_{11} &
(U_4^\top)_{11}\,Q_{12}\,(U_1^\top)_{22} &
(U_4^\top)_{11}\,Q_{13}\,(U_1^\top)_{33} \\
(U_4^\top)_{22}\,Q_{21}\,(U_1^\top)_{11} &
(U_4^\top)_{22}\,Q_{22}\,(U_1^\top)_{22} &
(U_4^\top)_{22}\,Q_{23}\,(U_1^\top)_{33} \\
(U_4^\top)_{33}\,Q_{31}\,(U_1^\top)_{11} &
(U_4^\top)_{33}\,Q_{32}\,(U_1^\top)_{22} &
(U_4^\top)_{33}\,Q_{33}\,(U_1^\top)_{33}
\end{pmatrix} = P U_3 P U_2 P =: V.
\end{equation*}

Because the diagonal blocks \((U_4^\top)_{ii}\) and \((U_1^\top)_{jj}\) are orthogonal matrices, we use the conjugation relation $Q_{ij} = U_4^{(i)} \, V_{ij} \, U_1^{(j)}.$ This two-sided orthogonal equivalence will be our main tool. Because each \(2\times 2\) orthogonal block is expressed via sines and cosines, the blocks are quite strongly related. Making this precise is somewhat technical; the details are given below.

We now explicitly compute the $2 \times 2$ blocks of the core matrix $V = P U_3 P U_2 P$, where $U_3 = \operatorname{diag}(W_1, W_2, W_3), \, U_2 = \operatorname{diag}(Y_1, Y_2, Y_3)$ with $W_i, Y_i \in \mathrm{O}(2)$. We parametrize $2 \times 2$ orthogonal blocks as follows
\[
W_i = \begin{pmatrix} a_i & b_i \\ c_i & d_i \end{pmatrix}, \qquad Y_i = \begin{pmatrix} e_i & f_i \\ g_i & h_i \end{pmatrix}, \quad \text{for } i \in \{1,2,3\}.
\]
Note that \( |a_i| = |d_i| \), \( |b_i| = |c_i| \), \( |e_i| = |h_i| \), \( |f_i| = |g_i| \), $|a_i|^2+|b_i|^2=1, \, |e_i|^2+|f_i|^2=1$ (these identities follow directly from the orthogonality of the \(2\times2\) blocks).

For \(P=P_{(2,3)}\), the corresponding zero-based index map is $\sigma=(0,3,1,4,2,5).$ Using this permutation, the two matrices entering the product
\(V=(P U_3 P)(U_2P)\) are
\[
P U_3 P=
\begin{pmatrix}
a_1&0&b_1&0&0&0\\
0&b_2&0&0&a_2&0\\
0&0&0&a_3&0&b_3\\
c_1&0&d_1&0&0&0\\
0&d_2&0&0&c_2&0\\
0&0&0&c_3&0&d_3
\end{pmatrix},
\qquad
U_2P=
\begin{pmatrix}
e_1&0&f_1&0&0&0\\
g_1&0&h_1&0&0&0\\
0&f_2&0&0&e_2&0\\
0&h_2&0&0&g_2&0\\
0&0&0&e_3&0&f_3\\
0&0&0&g_3&0&h_3
\end{pmatrix}.
\]
Therefore,
\begin{equation} \label{eq:gs_m4_middle_part}
V=P U_3 P U_2 P=
\begin{pmatrix}
a_1e_1&b_1f_2&a_1f_1&0&b_1e_2&0\\
b_2g_1&0&b_2h_1&a_2e_3&0&a_2f_3\\
0&a_3h_2&0&b_3g_3&a_3g_2&b_3h_3\\
c_1e_1&d_1f_2&c_1f_1&0&d_1e_2&0\\
d_2g_1&0&d_2h_1&c_2e_3&0&c_2f_3\\
0&c_3h_2&0&d_3g_3&c_3g_2&d_3h_3
\end{pmatrix}.   
\end{equation}

Only three \(2\times2\) blocks of \(V\) are needed:
\begin{equation} \label{eq:V_block_description}
V_{13}=
\begin{pmatrix}
b_1e_2&0\\
0&a_2f_3
\end{pmatrix},
\qquad
V_{22}=
\begin{pmatrix}
0&b_3g_3\\
c_1f_1&0
\end{pmatrix},
\qquad
V_{31}=
\begin{pmatrix}
d_2g_1&0\\
0&c_3h_2
\end{pmatrix}.
\end{equation}

Since \(Q_{ij}=U_4^{(i)}V_{ij}U_1^{(j)}\) and the outer blocks are orthogonal, \(Q_{ij}Q_{ij}^\top\) and \(V_{ij}V_{ij}^\top\) are orthogonally similar. We will use this observation to derive a contradiction from these three blocks alone.

We now want to construct \(Q \in \mathcal{A}_4 \cap \mathrm{O}(6)\), in other words, we impose $Q Q^\top = B_4 C B_1 B_1^\top C^\top B_4^\top = I_6,$ which is equivalent to $C B_1 B_1^\top C^\top = B_4^{-1} B_4^{-\top}$. Introducing \( H := B_1 B_1^\top \) and \( G := B_4^{-1} B_4^{-\top} \) (so that \( B_1 = H^{1/2} \) and \( B_4 = G^{-1/2} \)), we can reformulate the problem as finding positive definite block-diagonal matrices \( H \) and \( G \) satisfying \( C H C^\top = G \).
In fact, such matrices are found by solving systems of equations (and inequalities); the solutions constitute an entire family. For the sake of brevity, we present only a single example of such matrices.

Let $A(s):=\begin{pmatrix}1&1\\1&s\end{pmatrix}$ and set
$B_2=\operatorname{diag}\left(A(-3),A(-1/3),A(1/3)\right)$ and $B_3=\operatorname{diag}\left(A(1/3),A(-1/3),A(-3)\right)$.
Then substituting into \eqref{eq:gs_m4_middle_part}, we obtain
\[
C:=PB_3PB_2P=
\begin{pmatrix}
1&1&1&0&1&0\\
1&0&-3&1&0&1\\
0&-\frac13&0&1&1&\frac13\\
1&\frac13&1&0&\frac13&0\\
-\frac13&0&1&1&0&1\\
0&-\frac13&0&-3&1&-1
\end{pmatrix}.
\]
Define
\[
H_1=\begin{pmatrix}\frac{7}{4}&-3\\-3&\frac{21}{4}\end{pmatrix},\qquad
H_2=\begin{pmatrix}\frac{1}{4}&\frac{1}{3}\\\frac{1}{3}&\frac{1}{2}\end{pmatrix},\qquad
H_3=\begin{pmatrix}\frac{3}{4}&-1\\-1&\frac{3}{2}\end{pmatrix},
\]
and \(H:=\operatorname{diag}(H_1,H_2,H_3)\). $H_i$ are positive definite, and
\[
CHC^\top=G:=\operatorname{diag}(G_1,G_2,G_3),
\]
where
\[
G_1=\begin{pmatrix}2&-\frac{8}{3}\\-\frac{8}{3}&4\end{pmatrix},\quad
G_2=\begin{pmatrix}\frac{4}{3}&\frac{8}{9}\\\frac{8}{9}&\frac{2}{3}\end{pmatrix},\quad
G_3=\begin{pmatrix}\frac{28}{9}&-\frac{16}{3}\\-\frac{16}{3}&\frac{28}{3}\end{pmatrix}.
\]
\(G\) is also positive definite and invertible. Set $B_1:=H^{1/2}, B_4:=G^{-1/2}, Q:=B_4CB_1.$ It remains to show \(Q\notin\mathcal{A}_4^{\mathrm{orth}}\). Note that
$C_{13}=I_2, C_{22}=\begin{pmatrix}0&1\\1&0\end{pmatrix}, C_{31}=-\frac{1}{3}I_2.$ Since \(Q_{ij}=G_i^{-1/2}C_{ij}H_j^{1/2}\),
\[
Q_{13}Q_{13}^\top=Q_{22}Q_{22}^\top=\frac{3}{8}I_2,\qquad
Q_{31}Q_{31}^\top=\frac{1}{16}I_2.
\]

Now assume \(Q=U_4PU_3PU_2PU_1\) with \(U_i\in[\mathrm{O}(2)]^3\), and set \(V:=PU_3PU_2P\). Since \(Q_{ij}=U_4^{(i)}V_{ij}U_1^{(j)}\), we have 
$Q_{ij}Q_{ij}^\top = U_4^{(i)} \, V_{ij} V_{ij}^\top \, (U_4^{(i)})^\top = \mathrm{const * } I_2$, i.e.
$V_{13}V_{13}^\top=V_{22}V_{22}^\top=\frac{3}{8}I_2, 
V_{31}V_{31}^\top=\frac{1}{16}I_2.$ Using \eqref{eq:V_block_description}, we obtain
\[
|b_1 e_2|=|b_3 g_3| = |a_2 f_3| = |c_1 f_1| =\sqrt{\frac{3}{8}},\qquad
|c_3 h_2|= |d_2 g_1|= \frac{1}{4}.
\]
But orthogonality of \(W_3,Y_2\) gives $|b_3|=|c_3|, |e_2|=|h_2|,$ while \(|b_1|,|g_3| \le 1\). Hence we obtain a contradiction
\[
\frac{3}{8} = |b_1 e_2|\,|b_3 g_3|
= |b_1|\,|e_2|\,|b_3|\,|g_3|
= |b_1|\,|h_2|\,|c_3|\,|g_3|
=
|b_1 g_3|\,|c_3 h_2|
\le
\frac{1}{4},
\]
Therefore $Q \notin \mathcal{A}_4^{\mathrm{orth}}$, i.e.
$\mathcal{A}_4^{\mathrm{orth}} \neq \mathcal{A}_4 \cap \mathrm{O}(6).$ For conciseness, the explicit matrix multiplications are omitted; verification of the claimed properties can be found in the file \texttt{gs\_testing.ipynb}.

We now demonstrate how a decomposition consisting of \(m\) matrices can be extended to one consisting of \(m+1\) matrices while preserving the same properties. Although this construction is somewhat artificial, it demonstrates that different variants may occur in higher-order decompositions.

Let $P_{m} := \Pi_{12} =
\begin{pmatrix}
0&I_p&0\\
I_p&0&0\\
0&0&I_{N-2p}
\end{pmatrix} \neq I_N$ (with $N=pq$) be the permutation matrix interchanging the first two block coordinates. Define an $(m+1)$-factor family by setting \(P_m=\Pi_{12}\) and requiring the new factor \(B_{m+1}\) to have the same block structure as \(B_m\). Then $\mathcal A_{m+1}
= \Pi_{12}\mathcal A_m, \mathcal A_{m+1}^{\mathrm{orth}}
= \Pi_{12}\mathcal A_m^{\mathrm{orth}}.$ Moreover, \(\mathcal A_{m+1}^{\mathrm{orth}}\) is a manifold
if and only if \(\mathcal A_m^{\mathrm{orth}}\) has the same property.

We now prove this claim. Let $B_m = \operatorname{diag} \left(A^{(1)},A^{(2)},\ldots,A^{(q)} \right)$
be an arbitrary admissible leftmost factor. Conjugation by \(\Pi_{12}\) simply interchanges its first two blocks. Indeed,
on the first two block coordinates,
\[
\begin{pmatrix}
0&I_p\\
I_p&0
\end{pmatrix}
\begin{pmatrix}
A^{(1)}&0\\
0&A^{(2)}
\end{pmatrix}
\begin{pmatrix}
0&I_p\\
I_p&0
\end{pmatrix}
=
\begin{pmatrix}
A^{(2)}&0\\
0&A^{(1)}
\end{pmatrix},
\]
whereas all remaining blocks are unchanged. Hence \\ $\Pi_{12}^{\top}B_m\Pi_{12} = \operatorname{diag} \left(A^{(2)},A^{(1)}, A^{(3)},\ldots,A^{(q)}\right),$ which has exactly the same block structure as \(B_m\). Now consider an arbitrary matrix in the extended family $A_{m+1} = B_{m+1}\Pi_{12}B_mP_{m-1}\cdots P_1B_1.$ $B_{m+1}\Pi_{12} B_m = \Pi_{12} \left(\Pi_{12}^{\top}B_{m+1}\Pi_{12}\right)B_m$ since \(\Pi_{12}^2=I_N\).
Both matrices inside the parentheses and \(B_m\) are block diagonal with the same \(p\times p\) block structure. Their product therefore has the
same structure. Denoting this product by \(\widetilde B_m\), we obtain
\[
A_{m+1} =
\Pi_{12}\widetilde B_mP_{m-1}\cdots P_1B_1
\in
\Pi_{12}\mathcal A_m.
\]
Thus $\mathcal A_{m+1}\subseteq\Pi_{12}\mathcal A_m.$ Conversely, for every $A=B_m P_{m-1}\cdots P_1B_1\in\mathcal A_m,$ we may choose \(B_{m+1}=I_N\). Then $\Pi_{12}A = B_{m+1}\Pi_{12}B_mP_{m-1}\cdots P_1B_1
\in\mathcal A_{m+1}.$ Hence $\mathcal A_{m+1}=\Pi_{12}\mathcal A_m.$

Exactly the same argument applies to orthogonal factors. Indeed, if
\(B_{m+1}\) and \(B_m\) are block-diagonal orthogonal matrices, then $\left(\Pi_{12}^{\top}B_{m+1}\Pi_{12} \right)B_m$ is again block-diagonal and orthogonal. Therefore $\mathcal A_{m+1}^{\mathrm{orth}} = \Pi_{12}\mathcal A_m^{\mathrm{orth}}.$
Since \(\Pi_{12}\) is orthogonal, $\Pi_{12}A\in\mathrm{O}(N) \Longleftrightarrow A\in\mathrm{O}(N),$ and consequently $\mathcal A_{m+1}\cap\mathrm{O}(N)
= \Pi_{12}\left(\mathcal A_m\cap\mathrm{O}(N)\right).$
Together with $\mathcal A_{m+1}^{\mathrm{orth}} = \Pi_{12}\mathcal A_m^{\mathrm{orth}},$ this proves the claimed equivalence.

Finally, left multiplication $M\longmapsto\Pi_{12}M$
is a linear isometry of \(\mathbb R^{N\times N}\), since
$\|\Pi_{12}M-\Pi_{12}M_0\|_F = \|M-M_0\|_F.$

Hence it is an ambient diffeomorphism mapping
\(\mathcal A_m^{\mathrm{orth}}\) onto
\(\mathcal A^{\mathrm{orth}}_{m+1} \) and \(\mathcal A_m \cap\mathrm{O}(N)\) onto \(\mathcal A_{m+1} \cap\mathrm{O}(N)\). Therefore their local topological and smooth structures are identical.

\begin{corollary}
For every \(m \ge 4\), there exist decompositions with $\mathcal A_{m}^{\mathrm{orth}} = \mathcal A_{m}\cap \mathrm{O}(N),$ and others where the equality fails. The same flexibility holds for the manifold structure of $\mathcal A_{m}^{\mathrm{orth}}$ and $\mathcal A_{m}\cap \mathrm{O}(N)$.
\end{corollary}

\begin{remark}
Given \(m\), permutation matrices and block sizes, the proof or disproof of topological properties, the sufficiency of orthogonal blocks and related questions -- becomes cumbersome. It is unclear whether there is a way to simplify the arguments.
\end{remark}

\begin{remark}
\label{rem:gs_generalizations}
We note that we have repeatedly relied on the fact that the matrix \(P^\top (\cdot) P\) (where \(P\) is a Perfect Shuffle permutation) possesses a block structure with diagonal blocks. This property, however, is not exclusive to the Perfect Shuffle: for instance, one may first shuffle the elements between the blocks and then apply the Perfect Shuffle. Furthermore, the CS decomposition (and hence the hierarchical \(\mathcal{GS}\) decompositions) also holds in the unitary case, and Proposition~\ref{prop:m3}, the counterexamples, and so on, extend naturally to some other permutation matrices and/or unitary blocks.
\end{remark}

\section{Invertible blocks} \label{appendix:invertible}
Following the orthogonal-block case \cite{OF}, we can define a group action that yields an immersed submanifold (see Problem 21‑17 in \cite{Lee}), which is not necessarily embedded. Although it may still be embeddable, our tools only guarantee an immersion. This means that the topology induced on the set may differ from that of \(\mathrm{GL}(N)\); nevertheless, it remains a submanifold. Also observe that the set \(P_1^\top \operatorname{diag}(\lambda_1, \dots, \lambda_N) P_1 = \operatorname{diag}(\lambda_{\pi(1)}, \dots, \lambda_{\pi(N)})\) is contained in \(\mathcal{B}_2^{\times} \). This implies that the dimension of \(\text{Stab}_G(P_1)\) is at least \(N\). Hence, under the condition \(k_1 \ge b_2\), the perfect shuffle matrix also maximizes the dimension of the decomposition set in the case of invertible blocks.

This means that for \(m = 2\) we can also run RGD as in the orthogonal-block case, but such an extension lies beyond the scope of this paper.

\section{Generalization to Other Compact Matrix Groups}
\label{appendix:compact_groups}

The Riemannian framework extends naturally to some other compact Lie groups with only minor modifications. For \(\mathrm{SO}(b)\), the tangent space and gradient projection are identical to the orthogonal case \(\mathrm{O}(b)\); the determinant \(+1\) is preserved by several retractions provided the initial point lies in \(\mathrm{SO}(b)\). For the unitary group \(\mathrm{U}(b)\), the Riemannian metric uses the real part of the Frobenius inner product, and the orthogonal projection replaces skew‑symmetric matrices with skew‑Hermitian ones: \(\operatorname{skew}(A)=\frac{1}{2}(A-A^*)\). Note that in \(\mathrm{U}(b)\), an overlap is always present, the \(\mathrm{U}(b)\) case uses the same algorithmic structure; these modifications are straightforward and do not alter the method's core efficiency or structural design. Mathematical and implementation details are available in the library's source code.

\bibliographystyle{siamplain}
\bibliography{references}
\end{document}